\documentclass{article}
\usepackage{ijcai26}

\usepackage{times}
\usepackage{soul}
\usepackage{url}
\usepackage[hidelinks]{hyperref}
\usepackage[utf8]{inputenc}
\usepackage[small]{caption}
\usepackage{graphicx}
\usepackage{amsmath}
\usepackage{amsthm}
\usepackage{adjustbox}
\usepackage{amssymb} 
\usepackage{booktabs}
\usepackage{algorithm}
\usepackage{algorithmic}
\usepackage{subcaption}
\usepackage{multirow} 
\usepackage[switch]{lineno}
\usepackage{pifont}

\newtheorem{theorem}{Theorem}

\newcommand{\best}[1]{\textbf{#1}}

\title{SPAMoE: Spectrum-Aware Hybrid Operator Framework for Full-Waveform Inversion}

\author{
Zhenyu Wang$^{1*}$
\and
Peiyuan Li$^{1*}$
\and
Yongxiang Shi$^{1*}$
\And
Ruoyu Wu$^2$
\and
Chenfei Liao$^3$
\and
Lei Zhang$^{1\dagger}$\\
\affiliations
$^1$China University of Mining and Technology - Beijing\\
$^2$City University of Hong Kong (Dongguan)\\
$^3$The Hong Kong University of Science and Technology (Guangzhou)\\
\emails
2310640139@student.cumtb.edu.cn,
18991960125@163.com,
2310730419@student.cumtb.edu.cn,
72520100@cityu-dg.edu.cn,
cliao127@connect.hkust-gz.edu.cn,
zhanglei@cumtb.edu.cn\\
$^*$Equal contribution. $^\dagger$Corresponding author.
}

\begin{document}

\maketitle

\begin{abstract}
Full-waveform inversion (FWI) is pivotal for reconstructing high-resolution subsurface velocity models but remains computationally intensive and ill-posed. While deep learning approaches promise efficiency, existing Convolutional Neural Networks (CNNs) and single-paradigm Neural Operators (NOs) struggle with one fundamental issue: frequency entanglement of multi-scale geological features. To address this challenge, we propose \textbf{Spectral-Preserving Adaptive MoE (SPAMoE)}, a novel spectrum-aware framework for solving inverse problems with complex multi-scale structures. Our approach introduces a Spectral-Preserving DINO Encoder that enforces a lower bound on the high-to-low frequency energy ratio of the encoded representation, mitigating high-frequency collapse and stabilizing subsequent frequency-domain modeling. Furthermore, we design a novel Spectral Decomposition and Routing mechanism that dynamically assigns frequency bands to a Mixture-of-Experts (MoE) ensemble comprising FNO, MNO, and LNO. On the ten OpenFWI sub-datasets, experiments show that SPAMoE reduces the average MAE by \textbf{44.4\%} relative to the best officially reported OpenFWI baseline, thereby establishing a new architectural framework for learning-based full-waveform inversion. Our code and data are available at \url{https://github.com/zhenyuwang12366/SPAMoE}
\end{abstract}

\section{Introduction}

\begin{figure}[t]
    \centering
    \includegraphics[width=0.9\linewidth]{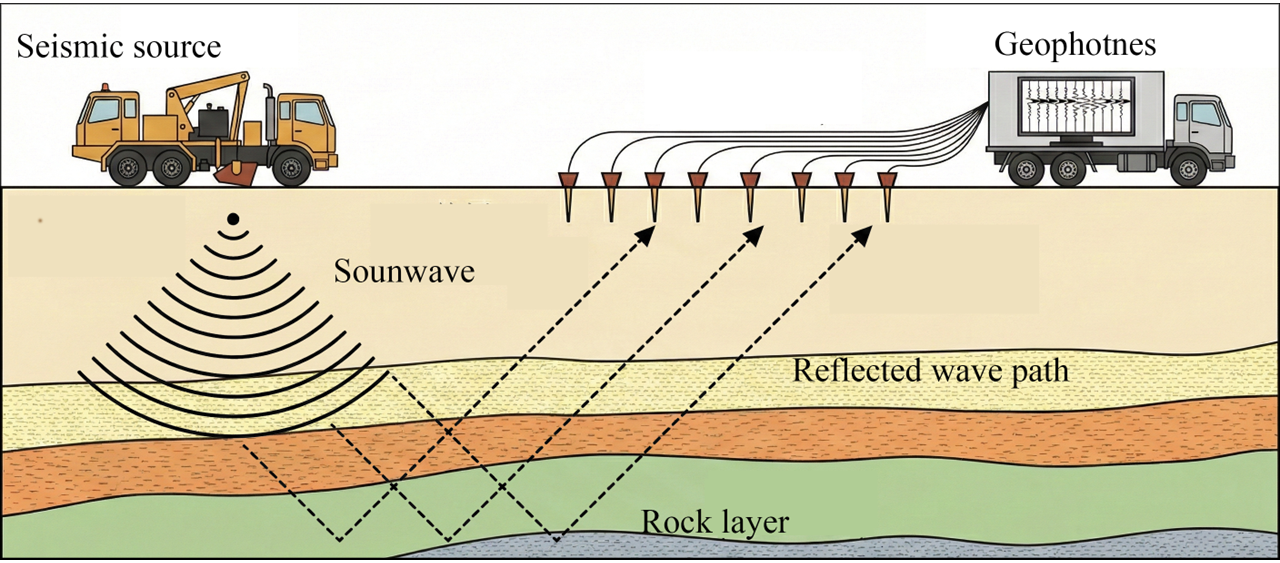}
    \caption{
    \textbf{Schematic illustration of seismic full-waveform Inversion (FWI).}
    }
    \label{fig:intro}
\end{figure}

As shown in Figure~\ref{fig:intro}, seismic full-waveform inversion (FWI) has become an important technique in modern geophysics~\cite{virieux2009overview}. 
FWI is essentially a highly nonlinear and ill-posed inverse problem, aiming to recover high-resolution subsurface physical parameter fields, such as velocity models, from observed seismic wavefields. 
As exploration targets become increasingly complex and the demand for imaging accuracy continues to rise, FWI exhibits stronger capability than traditional velocity analysis and traveltime tomography in characterizing complex geological structures. 
However, conventional physics-constrained iterative inversion methods have long been limited by cycle-skipping, high computational cost, and strong sensitivity to the initial model, which is particularly pronounced in high-resolution and structurally complex scenarios~\cite{geng2018frequency}. 
These challenges have motivated growing interest in data-driven alternative modeling paradigms. 

In recent years, advances in deep learning for scientific computing have propelled research on data-driven FWI~\cite{wu2019inversionnet,zhang2019velocitygan}, aiming to achieve a more practical trade-off between computational efficiency and optimization stability. 
Among these approaches, neural operators (NOs)~\cite{kovachki2023neural} learn mappings between PDE solutions in function spaces and exhibit resolution-invariant modeling capability, providing a promising path to rapidly approximate wave-equation-related operators and build end-to-end inversion models. 
From a spectral perspective, the multi-scale information in FWI shows clear frequency dependence: low-frequency components mainly constrain large-scale background velocities and macroscopic structures, providing a more stable global trend for inversion; whereas high-frequency components are more sensitive to fine geological features such as faults, thin layers, and sharp interfaces, largely determining the resolution and detail quality of the final imaging. 
Nevertheless, existing neural-operator-based FWI methods typically process information from different frequency bands through a single pathway, which makes information from different frequency bands prone to becoming entangled and interfering with each other during learning in multi-scale scenarios. 
As a result, the recovery of fine geological details is limited, and it becomes difficult to achieve simultaneously strong performance on both global backgrounds and local details. 
Therefore, how to explicitly decouple and specifically handle inversion information from different frequency bands within a learning framework remains a key yet insufficiently addressed problem. 

To tackle the coupling of multi-scale components in the velocity model---from smooth backgrounds to sharp faults---in the frequency domain, and to enable effective modeling after frequency decoupling, we propose \textbf{SPAMoE} (Spectral-Preserving Adaptive MoE), a spectrum-aware framework for learning-based FWI. 
\textbf{First}, SPAMoE employs a Spectral-Preserving DINO~\cite{simeoni2025dinov3} Encoder. Beyond aligning waveform observations to a structure-consistent latent representation, this encoder enforces a lower bound on the prediction’s high-to-low frequency energy ratio (HL). By mitigating high-frequency collapse and maintaining balanced frequency content, it provides a reliable foundation for subsequent frequency-domain modeling in the MoE module. 
\textbf{Second}, SPAMoE introduces an Adaptive Spectral Mixture-of-Experts consisting of three components: Concentric Soft Frequency-Band Decomposition, an Adaptive Frequency-Preference Mechanism, and a Spectral Energy Attention Router. 
Together, these components establish a complete data flow: frequency decoupling via soft-band decomposition, adaptive band allocation guided by frequency preference, and dynamic activation of complementary experts based on global spectral-energy patterns. 
Through these designs, SPAMoE organically connects ``alignment--decoupling--modeling--learning'' within a unified framework, providing a more robust pathway for high-resolution inversion in complex geological settings.

We conduct systematic evaluations of SPAMoE on the OpenFWI~\cite{deng2022openfwi} benchmark. 
Experimental results show that SPAMoE yields substantial improvements over single neural operators and multiple mainstream learning-based inversion baselines on FWI, with particularly stable recovery of complex structures and high-frequency details (e.g., faults and sharp interfaces). 
Moreover, SPAMoE also demonstrates strong performance on the pipe flows task (see the supplementary material~\ref{supply:pipe}), suggesting that the proposed ``spectral decoupling--expert specialization--adaptive routing'' modeling strategy has certain generality and future potential.

The key contributions of this work are threefold:
\begin{itemize}
    \item We propose a Spectrum-Aware Hybrid Neural Operator framework (SPAMoE). This framework explicitly decouples high and low-frequency information flows, effectively alleviating the frequency coupling inherent in traditional end-to-end models.
    \item We design two core modules: a Spectral-Preserving DINO Encoder to maintain balanced frequency content, and an Adaptive Spectral Mixture-of-Experts that performs frequency decomposition, routing, and operator modeling to improve multi-scale geological structure reconstruction.
    \item We evaluate SPAMoE on all ten official OpenFWI sub-datasets and show that it consistently outperforms the official OpenFWI baselines, reducing the averaged MAE by 44.4\% relative to the strongest baseline.
\end{itemize}


\section{Related Work}
\label{sec:relatedwork}
\subsection{Encoder-Neural Operator Architectures}
For complex PDE tasks, single neural operators are often insufficient. Consequently, researchers have designed encoder-neural operator architectures that integrate feature extraction capabilities: for instance, U-NO~\cite{rahman2022u} and U-FNO~\cite{wen2022ufno} introduce convolutional encoders to enhance multi-scale feature aggregation; MINO~\cite{shi2025mesh} leverages Transformers to capture irregular mesh geometry; while VANO~\cite{seidman2023variational} and DA-NO~\cite{viknesh2025differentiable} further explore variational modeling in the latent space. Although these architectures demonstrate potential in handling complex PDEs, FWI remains particularly challenging due to its multi-scale nature, which leads to complex spectral content and strong cross-band coupling. In practice, generic encoders may not explicitly preserve such spectral balance. 
We introduce a Spectral-Preserving DINO Encoder that enforces a lower bound on the high-to-low frequency energy ratio of the encoded representation, helping prevent high-frequency collapse during encoding and yielding a more spectrum-faithful latent space for subsequent operator learning.

\subsection{Neural Operators in FWI and MoE in PDE}
Neural operators have been applied in seismic wavefield simulation~\cite{yang2023rapid,zou2025ambient} and direct inversion~\cite{zhu2023fourier}. However, existing methods rely on a single operator pathway, leading to frequency coupling during optimization. In deep learning-based PDE solving, MoE has been employed in Physics-Informed Neural Networks~\cite{bischof2022mixture}, soft domain decomposition~\cite{hu2023augmented}, DeepONet ensembles~\cite{sharma2024ensemble}, and boundary condition learning with model selection~\cite{deighan2025mnoe}. Nevertheless, most existing MoE methods for PDEs are based on spatial domain decomposition. For FWI, the core challenge lies in the complexity of the spectral dimension~\cite{virieux2009overview}, rendering these direct spatial decomposition methods unsuitable for FWI. Meanwhile, the excellent performance of FreqMoE~\cite{chen2025freqmoe} demonstrates the superiority of routing using spectral features. However, it relies solely on spectral features and lacks perception of the global spectrum. The Adaptive Spectral MoE framework proposed in this paper explicitly decouples high- and low-frequency information flows and utilizes an attention mechanism to dynamically activate complementary operator experts. This fills the gap in dynamic routing for multi-path complementary neural operators based on global perception of spectral energy.

\section{Methodology}
\label{sec:method}

\subsection{Overall Framework}
\label{subsec:overview}

\begin{figure*}[t]
    \centering
    \includegraphics[width=\linewidth]{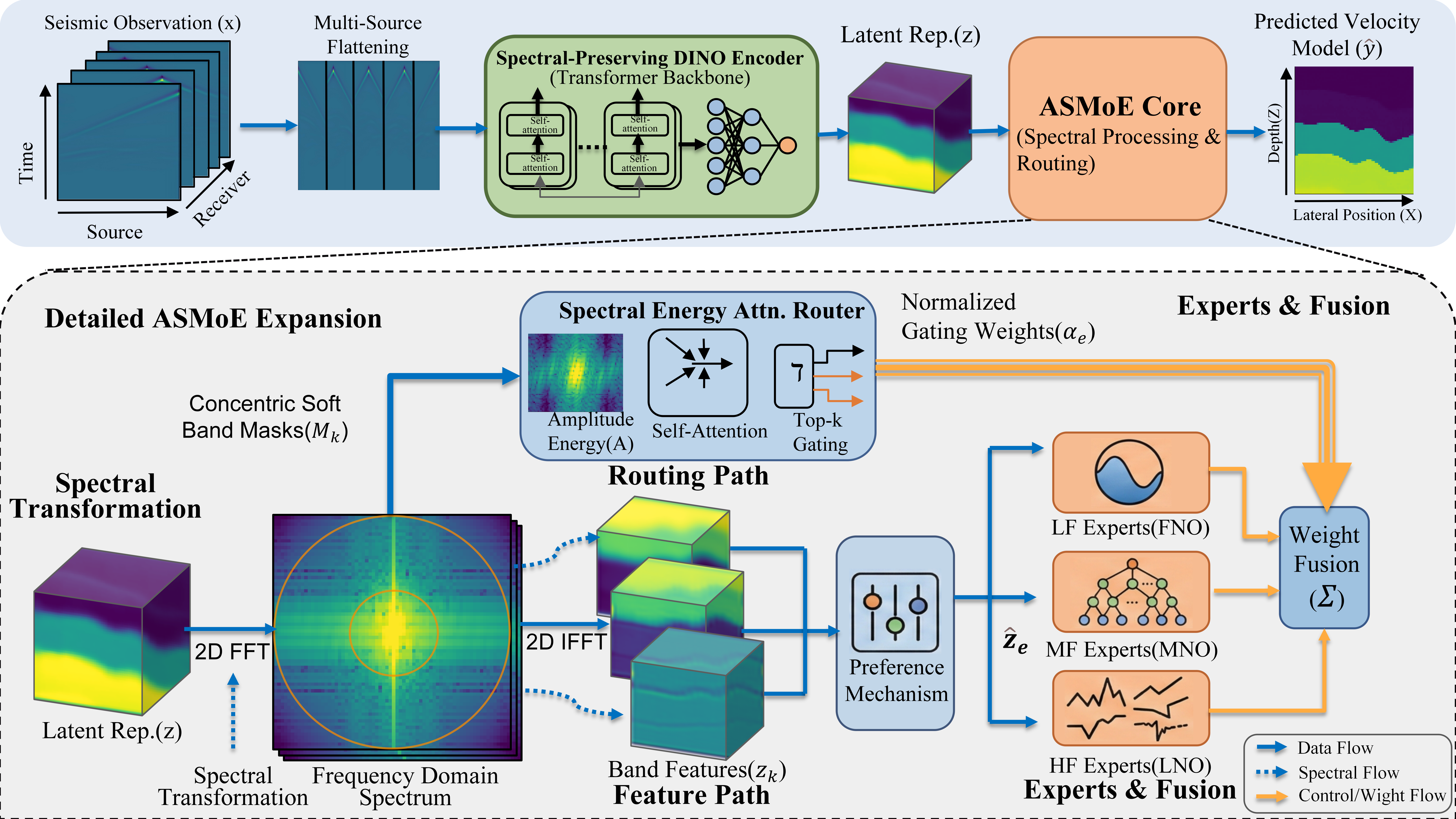}
    \caption{Overview of the SPAMoE framework. (a) The Spectral-Preserving DINO Encoder projects time-receiver domain seismic observations into spatially aligned latent representations; (b) Adaptive Spectral MoE uses the spectral energy distribution as routing features to activate experts, decomposes the full-spectrum representation into multiple concentric soft frequency bands, and adaptively fuses the band-wise features as inputs to the experts.}
    \label{fig:framework}
\end{figure*}

As illustrated in Figure~\ref{fig:framework}, we propose the \textbf{Spectral-Preserving Adaptive MoE} (SPAMoE) framework. The framework consists of two synergistic core modules: (Section~\ref{subsec:encoder}) a Spectral-Preserving DINO Encoder, and (Section~\ref{subsec:afmoe}) an Adaptive Spectral Mixture-of-Experts model, including spectral decomposition and routing mechanisms.

The overall workflow of SPAMoE is as follows: First, the Spectral-Preserving DINO Encoder maps the observations $x$ in the time-receiver domain to a spatially aligned latent representation $z$. Then, the Adaptive Spectral MoE performs differentiable spectral decomposition and routing decisions on $z$ in the frequency domain, sparsely activates a set of complementary expert operators, and finally produces the predicted velocity model $\hat{y}$.

\subsection{Preliminaries}
\label{subsec:preliminaries}
\noindent\textbf{Problem Definition of Full-Waveform Inversion.}  FWI can be formulated as an ill-posed inverse scattering problem, whose goal is to reconstruct the subsurface velocity model from seismic wavefields observed at the surface. Given observations
$x \in \mathbb{R}^{N_s \times T \times N_r}$,
where $N_s$, $T$, and $N_r$ denote the number of sources, the number of temporal samples, and the number of receivers, respectively, our goal is to reconstruct the velocity model in the spatial domain
$y \in \mathbb{R}^{H \times W}$.

To this end, we learn a nonlinear mapping $\Phi_\theta: x \rightarrow y$ by minimizing the reconstruction error:
\begin{equation}
    \theta^* = \operatorname*{argmin}_\theta
    \mathbb{E}_{(x, y)}
    \left[
        \mathcal{L}(\Phi_\theta(x), y)
    \right],
\end{equation}
where $\mathcal{L}$ denotes the objective function measuring the discrepancy between the predicted and ground-truth velocity models.

\noindent\textbf{Frequency-Domain Analysis and Energy Metrics.} To analyze spectral characteristics of physical fields, we consider the centered 2D discrete Fourier transform
$\widehat{u} = \mathcal{U}(u) \in \mathbb{C}^{H \times W}$
(see Appendix~\ref{app:freq_def}1 for details). 
We partition the frequency domain according to the normalized radial frequency $r(\omega)\in[0,1]$ (see Appendix~\ref{app:grid_def}2 for details).
Let $\Omega_L$ and $\Omega_H$ denote the low- and high-frequency sets, respectively:
\begin{equation}
    \Omega_L = \{\omega \mid r(\omega) < r_{split}\},
    \Omega_H = \{\omega \mid r(\omega) \ge r_{split}\}.
\end{equation}
To quantify spectral deviation, we define the low-/high-frequency spectral energies $E_L(u)$ and $E_H(u)$, as well as the high-to-low frequency energy ratio $\mathrm{HL}(u)$ (see Appendix~\ref{app:energy_def}3 for details). This metric serves as a core indicator in our subsequent theoretical analysis.
\subsection{Spectral-Preserving DINO Encoder}
\label{subsec:encoder}

To cope with the spectral complexity of FWI, we adopt a Spectral-Preserving DINO Encoder. This encoder not only aligns the waveform observations with the spatial velocity representation, but also establishes a lower bound on the high-to-low frequency energy ratio (HL) of the encoder output, helping keep the frequency content balanced and providing a reliable foundation for subsequent frequency-domain operations in the MoE module.

\noindent\textbf{Multi-Source Observation Reorganization and Network Implementation.} For the observation tensor $x$, different source-excited wavefields can be viewed as multiple independent observations of the same subsurface medium, exhibiting intrinsic spatial correlations along the receiver dimension. Instead of treating each shot gather independently as batch samples, we explicitly concatenate the source dimension $N_s$ into the receiver dimension $N_r$ to form a unified panoramic observation matrix:
\begin{equation}
    x' = \operatorname{Reshape}(x) \in \mathbb{R}^{T \times (N_s \cdot N_r)}.
\end{equation}
This transformation flattens the original 3D tensor into a 2D global observation plane, enabling the model to aggregate scattering signatures across all shots jointly. We then feed $x'$ into a Vision Transformer backbone pre-trained via self-supervision (DINO) to obtain the latent representation:
\begin{equation}
    z = E_\theta(x') \in \mathbb{R}^{C \times H \times W}.
\end{equation}
This design ensures that $z$ is geometrically aligned with the target model $y$, providing a unified spatial input to the subsequent spectral MoE module.

\noindent\textbf{The Oversmoothing of Spatial Interpolation.} In handling multi-source observation alignment, naive spatial interpolation introduces severe oversmoothing. For common spatial interpolation operators such as bilinear interpolation, the frequency response typically behaves as a low-pass filter: low-frequency components are largely retained, whereas high-frequency components are more strongly attenuated. We model this property by assuming that the high-frequency response is uniformly upper-bounded by $\alpha$ and the low-frequency response is uniformly lower-bounded by $\beta$, with $0<\alpha<\beta\le 1$. As theoretically characterized below, interpolation imposes an upper bound on the HL ratio, thus damaging the high-frequency structure.

\begin{theorem}[Interpolation Imposes an Upper Bound on the HL Ratio]
\label{thm:interp_lowpass}
Let the interpolation operator $I:\mathbb{R}^{H\times W}\to\mathbb{R}^{H\times W}$ satisfy a multiplicative frequency-domain response:
\begin{equation}
\label{eq:response}
\widehat{I(u)}(\omega)=H_I(\omega)\widehat u(\omega).
\end{equation}
If there exist constants $0<\alpha<\beta\le 1$ such that $|H_I(\omega)|\le \alpha$ for $\forall \omega\in\Omega_H$ and $|H_I(\omega)|\ge \beta$ for $\forall \omega\in\Omega_L$, then for any $u\neq 0$, the following holds:
\begin{equation}
    \label{eq:thmA_bound}
    \mathrm{HL}(I(u))\le \frac{\alpha^2}{\beta^2}\,\mathrm{HL}(u).
\end{equation}
\end{theorem}
\begin{proof}
    See supplementary material~\ref{supply:proof_interp} for the complete proof.
\end{proof}

This theorem indicates that when interpolation attenuates high-frequency components more strongly than low-frequency components ($\alpha<\beta$), it strictly reduces the HL ratio. To overcome this limitation, we design a learning-based structural alignment.

\noindent\textbf{Theoretical Analysis of Spectral Preservation.} To theoretically ensure that the encoder does not induce systematic high-frequency collapse, we introduce a linear readout operator $R:\mathbb{R}^{C\times H\times W}\to\mathbb{R}^{H\times W}$, and define a comparable spatial field $u_c = R(E_\theta(x'))\in\mathbb{R}^{H\times W}$. The downstream prediction is written as $\hat y_c = F(u_c)$, where $F$ denotes the downstream operator. We adopt the following testable assumptions:

\textbf{A1: High-Frequency Non-Contractiveness. } The encoder preserves the dominant energy in the high-frequency band. That is, there exists $\delta \ge 1$ such that
\begin{equation}
    \label{eq:hf_noncontract}
    E_H(u_c) \ge \delta\,E_H(y),
\end{equation}
where $E_H$ denotes the spectral energy in the high-frequency band (see Appendix~\ref{app:energy_def} for details).

\textbf{A2: Controllable Low-Frequency Energy.} The low-frequency energy of the encoder output is of the same order of magnitude as that of the ground truth. That is, there exists $\kappa \ge 1$ such that
\begin{equation}
    \label{eq:lf_control}
    E_L(u_c) \le \kappa\,E_L(y),
\end{equation}
where $E_L$ denotes the spectral energy in the low-frequency band (see Appendix~\ref{app:energy_def} for details).

\textbf{A3: Boundedness of the Downstream Operator.} The amplification factors of the downstream prediction operator $F$ on different frequency bands are bounded. That is, there exist constants $0 < m \le M < \infty$ with $M\ge 1$ such that for each frequency band,
\begin{equation}
\label{eq:bounded_resp}
\begin{aligned}
    mE_H(u_c) &\le E_H(F(u_c)) \le ME_H(u_c), \\
    mE_L(u_c) &\le E_L(F(u_c)) \le ME_L(u_c).
\end{aligned}
\end{equation}

To validate the rationality of these assumptions, we compute empirical statistics on the test set (CurveVel-A) using FNO as the downstream operator. As summarized in Table~\ref{tab:spectral-analysis}, the DINO Encoder strongly preserves high-frequency energy on average ($\delta \approx 19$) and maintains a finite low-frequency bound ($\kappa \approx 35$), providing empirical support for assumptions A1--A3~\eqref{eq:hf_noncontract}--\eqref{eq:bounded_resp}. In contrast, the interpolation baseline severely collapses high-frequency content ($\mathrm{Ratio}_1 \approx 0$). Detailed metric definitions are provided in supplementary material~\ref{supply:valid}.

\begin{table}[htbp]
    \centering
    \small
    \setlength{\tabcolsep}{4pt}
    \renewcommand{\arraystretch}{1.15}
    \begin{tabular}{lccc}
        \toprule
        \textbf{Metric} & \textbf{Encoder} & \textbf{Interpolation} & \textbf{GT} \\
        \midrule
        $\mathrm{Ratio}_1$ ($E_H$ Preserv.) & 19.032 & $3.876 \times 10^{-6}$ & -- \\
        $\mathrm{Ratio}_2$ ($E_L$ Control) & 34.904 & $1.786 \times 10^{-7}$ & -- \\
        HL (Mean) & 0.105 & 0.084 & 0.111 \\
        $g_H$ (Order Range) & $[10^{-2},\;10^2]$ & $[10^{5},\;10^6]$ & -- \\
        $g_L$ (Order Range) & $[10^{-2},\;10^2]$ & $[10^{6},\;10^8]$ & -- \\
        \bottomrule
    \end{tabular}
    \caption{Empirical verification of spectral assumptions A1--A3. The proposed Spectral-Preserving DINO Encoder provides empirical support for the assumptions with bounded gains, whereas spatial interpolation exhibits severe high-frequency collapse and imbalanced frequency-band gains.}
    \label{tab:spectral-analysis}
\end{table}

Based on these empirically supported assumptions, we establish the following theorem to characterize the spectral preservation property of the overall framework:

\begin{theorem}[Spectral Preservation of the Spectral-Preserving DINO Encoder]
\label{main:thm}
Under assumptions A1--A3~\eqref{eq:hf_noncontract}--\eqref{eq:bounded_resp}, let the final prediction be $\hat y_c = F(u_c)$. Then for any sample, the high-to-low frequency energy ratio (HL) admits the following lower bound:
\begin{equation}
\label{eq:thm1_main}
\mathrm{HL}(\hat y_c) \;\ge\; \frac{m}{M}\cdot\frac{\delta}{\kappa}\cdot \mathrm{HL}(y).
\end{equation}
\end{theorem}

\textit{Proof Sketch.} Based on the boundedness of the downstream operator in assumption A3~\eqref{eq:bounded_resp}, the output high-frequency energy is lower-bounded by $m E_H(u_c)$, while the low-frequency energy is upper-bounded by $M E_L(u_c)$. Substituting the non-contractiveness assumption A1~\eqref{eq:hf_noncontract} and the controllability assumption A2~\eqref{eq:lf_control} into these bounds, we can directly map the predicted HL ratio to the ground truth HL ratio, scaled by the constant factor $\frac{m}{M}\cdot\frac{\delta}{\kappa}$. The complete rigorous proof is provided in supplementary material~\ref{supply:proof_dino}.

This theorem indicates that as long as the Spectral-Preserving DINO Encoder does not compress high-frequency components ($\delta \ge 1$) and maintains finite low-frequency amplification ($\kappa$), and the downstream MoE operator does not introduce unbounded band distortion, the HL ratio of the final prediction remains controlled by a constant of the same order as the ground truth $y$, thereby guaranteeing spectral fidelity.

\subsection{Adaptive Spectral Mixture-of-Experts}
\label{subsec:afmoe}

FWI velocity models typically contain both smooth backgrounds and sharp interfaces across multiple scales. A fixed single-path model often struggles to disentangle the frequency components associated with different scales. To address this issue, we propose an Adaptive Spectral MoE, which consists of Concentric Soft Frequency-Band Decomposition, Adaptive Frequency-Preference Mechanism, Spectral Energy Attention Router, and Complementary Neural-Operator Experts.

\noindent\textbf{Concentric Soft Frequency-Band Decomposition.}
\label{subsubsec:soft_bands} We adopt a differentiable frequency partition using Gaussian soft masks. In the centered spectral coordinate system, let the number of bands be $K$, and the center of the $k$-th band be $c_k=\frac{k-1}{K-1}$. We define the Gaussian concentric soft-band mask as:
\begin{equation}
    M_k(i,j) = \exp\left(-\gamma (r(i,j) - c_k)^2\right),
\end{equation}
where $\gamma$ denotes band sharpness. We then apply the mask in the frequency domain and perform an inverse transform to obtain the feature for each band:
\begin{equation}
    z_k=\mathcal{U}^{-1}(\widehat{z}\odot M_k) \in\mathbb{R}^{C\times H\times W}.
\end{equation}

\noindent\textbf{Adaptive Frequency-Preference Mechanism.}
\label{subsubsec:freq_pref} After obtaining band-wise features, strictly binding each expert to a fixed band limits flexibility. To allow each expert to adaptively select suitable frequency regions while retaining its inductive bias, we introduce a learnable frequency-preference parameter for each expert. Each expert thus receives a soft combination of $\{z_k\}$ rather than a single band.

For each expert $e\in\{1,\dots,N_E\}$, we define a learnable scalar $f_e\in[0,1]$. The mixing weights are computed based on its distance to the band center $c_k$:
\begin{equation}
    s_{e,k}=-\eta\,(f_e-c_k)^2,\quad
    \pi_{e,k}=\frac{\exp(s_{e,k})}{\sum_{t=1}^{K}\exp(s_{e,t})}.
\end{equation}
where $\eta$ denotes frequency-affinity sharpness. The input feature to expert $e$ is constructed as:
\begin{equation}
    \tilde{z}_e=\sum_{k=1}^{K}\pi_{e,k}z_k.
\end{equation}
With this mechanism, each expert is able to adaptively focus on the most suitable frequency components around its preferred band.

\noindent\textbf{Spectral Energy Attention Router.}
\label{subsubsec:spectral_router} 
The above two subsections define the expert inputs based on frequency-domain features. We further require a routing mechanism that is explicitly sensitive to the spectral characteristics of the input signal and aligns well with expert inputs. Therefore, we design a lightweight spectral attention-based router driven by spectral energy. The router only uses the spectral energy distribution to generate gating weights, while the latent feature $z$ retaining full phase information is delivered to the activated experts for processing. 
Moreover, in the FWI context, inter-sample spectral energy distributions are crucial indicators for distinguishing different geological structures~\cite{partyka1999interpretational}.

Specifically, we first compute the energy map of the centered spectrum,
$\mathbf{A} = \sqrt{P_z(\omega)} \in \mathbb{R}^{H \times W}$, where $P$ denotes the power spectrum (see Appendix~\ref{app:energy_def}3 for details).
To capture global spectral dependencies, we build a self-attention layer:
\begin{equation}
    Q,K,V=\phi_{\mathrm{qkv}}(\mathbf{A}),
\end{equation}
\begin{equation}
    \mathrm{Attention}(Q,K,V) = \mathrm{softmax}\left(\frac{\langle Q,\, K \rangle_C}{\sqrt{d_k}}\right)V,
\end{equation}
\begin{equation}
    \mathbf{A}' = \phi_{\mathrm{agg}}(\mathrm{Attention}(Q,K,V)),
\end{equation}
where $\phi_{\mathrm{qkv}}$ denotes a linear projection, $\langle \cdot \rangle_C$ denotes the channel-wise inner product, $\phi_{\mathrm{agg}}$ is an aggregation network, and $d_k$ is the scaling factor. The spectral attention mechanism aggregates global spectral-energy patterns and identifies dominant band characteristics.

We then map the aggregated spectral feature $\mathbf{A}'$ to expert gating scores $g \in \mathbb{R}^{N_E}$, and apply a Top-$k$~\cite{shazeer2017outrageously} strategy to generate sparse routing decisions:
\begin{equation}
    \mathcal{T}(x) = \operatorname{TopK}(g, N_k), \quad
    \alpha_e(x) = \frac{\exp(g_e)}{\sum_{j \in \mathcal{T}(x)} \exp(g_j)}.
\end{equation}
With this design, the router adaptively activates the most suitable combination of experts according to the spectral energy of each sample.

\noindent\textbf{Complementary Neural-Operator Experts.}
\label{subsubsec:experts} For the FWI setting, we construct an expert set with complementary inductive biases as follows:

Low-Frequency Expert (FNO)~\cite{lifourier}: It leverages global spectral convolutions to recover background velocity structures:
\begin{equation}
    \mathcal{O}_{FNO}(\tilde{z}) =
    \mathcal{F}^{-1}\left(W \odot \mathcal{F}(\tilde{z})\right).
\end{equation}

Mid-Frequency Expert (MNO)~\cite{lutjens2022multiscale}: It models transitional stratigraphic structures via hierarchical multi-scale convolutional kernels:
\begin{equation}
    \mathcal{O}_{MNO}(\tilde{z}) =
    \sum_{s=1}^{S} \phi_s(K_s * \tilde{z}).
\end{equation}

High-Frequency Expert (LNO)~\cite{li2024local}: It captures faults and sharp interfaces using position-dependent local operators:
\begin{equation}
    \mathcal{O}_{LNO}(\tilde{z})(x) =
    \int_{\Omega_x} K(x,y)\tilde{z}(y)\,dy,
\end{equation}
where $\Omega_x$ denotes the local neighborhood of location $x$, and the operator targets high-frequency local variations.

\noindent\textbf{MoE Fusion.}
\label{subsubsec:moefusion} The final output is obtained via sample-wise weighted fusion:
\begin{equation}
\hat{y}
=
\sum_{e \in \mathcal{T}(x)}
\alpha_e(x) \, \mathcal{O}_{e}(\tilde{z})
\end{equation}
where $\mathcal{T}(x)$ is the set of selected expert indices, $\alpha_e(x)$ denotes the gating weight, and $\mathcal{O}_e$ denotes the operator implemented by expert $e$.
%

\section{Experiments}

\subsection{Datasets}

We conduct systematic evaluations of SPAMoE on all ten official 2D sub-datasets of the OpenFWI benchmark, including CurveVel-A/B, FlatVel-A/B, CurveFault-A/B, FlatFault-A/B, and Style-A/B, and follow the official data splits and evaluation protocol provided by OpenFWI for training and evaluation. For all geological families, ``A'' versions correspond to smoother, lower-complexity structures, while ``B'' versions include more irregular layering, stronger nonlinearity, and higher-frequency reflections, posing greater challenges for recovering multi-scale details and high-frequency structures.

\noindent\textbf{Dataset composition.}
The Vel sub-datasets (FlatVel-A/B and CurveVel-A/B) contain 24k/6k training and validation samples; the Fault sub-datasets (FlatFault-A/B and CurveFault-A/B) provide 48k/6k samples; and the Style-A/B sub-datasets contain 60k/7k samples.
Each sample consists of five seismic shot gathers
$\mathbf{S}\in\mathbb{R}^{5\times1000\times70}$
and a ground-truth velocity map
$\mathbf{V}\in\mathbb{R}^{1\times70\times70}$,
where 1000 denotes the temporal sampling length and 70 denotes the number of receivers.
Following the OpenFWI input convention, we concatenate the five shots along the receiver axis to form a single-channel input
$\mathbf{S}'\in\mathbb{R}^{1\times1000\times350}$,
so that the model learns end-to-end from a unified input tensor.
The sample sizes and split settings are consistent with the official OpenFWI benchmark.

\noindent\textbf{Preprocessing.}
We apply a log transform to the seismic inputs to compress the dynamic range and improve numerical stability, followed by per-sub-dataset min--max normalization; the velocity maps are also scaled by per-sub-dataset min--max normalization. No additional augmentation is used.

\subsection{Experimental Setup}
\paragraph{Implementation.}
All models are implemented in PyTorch~2.8. We train the model with the AdamW optimizer and a warmup cosine schedule with restarts. We train one independent model per OpenFWI 2D sub-dataset. All models share the same architecture and loss design, while optimization hyperparameters follow a unified configuration (Table~\ref{tab:hyperparams}). Training is conducted on dual RTX~4090 GPUs with a per-GPU batch size of 32. Unless otherwise stated, all settings apply to all ten sub-datasets.

We report three standard image-level reconstruction metrics on the OpenFWI test sets: mean absolute error (MAE), root mean squared error (RMSE), and peak signal-to-noise ratio (PSNR). MAE and RMSE are computed pixel-wise between the predicted and ground-truth velocity maps (lower is better), while PSNR measures reconstruction fidelity in decibels (higher is better). Following the OpenFWI practice of reporting the best-performing checkpoint, we evaluate the metrics at each epoch during training and report the results from the epoch with the lowest MAE for each sub-dataset.

\noindent\textbf{Comparison methods.}
We compare our SPAMoE model against the three official OpenFWI baselines: InversionNet~\cite{wu2019inversionnet}, VelocityGAN~\cite{zhang2019velocitygan}, and UPFWI~\cite{jin2021unsupervised}, and additionally include FNO~\cite{lifourier} as a single neural-operator baseline to quantify the gains brought by our spectral decoupling and MoE design over a single-operator backbone. We choose InversionNet, VelocityGAN, and UPFWI because they 
are the most commonly used and paradigm-covering baselines in the OpenFWI benchmark~\cite{deng2022openfwi}, representing supervised CNN-based direct inversion (InversionNet), supervised GAN-based inversion (VelocityGAN), and physics-informed unsupervised inversion with differentiable forward modeling (UPFWI).

\begin{table*}[!t]
\centering
\setlength{\tabcolsep}{3pt}
\renewcommand{\arraystretch}{1.15}
\begin{adjustbox}{max width=\textwidth}
\begin{tabular}{ll ccc ccc ccc ccc ccc ccc}
\toprule
\multirow{2}{*}{\textbf{\textit{Family}}} & \multirow{2}{*}{\textbf{\textit{Subset}}} 
& \multicolumn{3}{c}{\textbf{\textit{InversionNet}}} 
& \multicolumn{3}{c}{\textbf{\textit{VelocityGAN}}} 
& \multicolumn{3}{c}{\textbf{\textit{UPFWI}}} 
& \multicolumn{3}{c}{\textbf{\textit{FNO}}} 
& \multicolumn{3}{c}{\textbf{\textit{Ours (SPAMoE)}}} \\
\cmidrule(lr){3-5}\cmidrule(lr){6-8}\cmidrule(lr){9-11}\cmidrule(lr){12-14}\cmidrule(lr){15-17}
& & MAE$\downarrow$ & RMSE$\downarrow$ & SSIM$\uparrow$
  & MAE$\downarrow$ & RMSE$\downarrow$ & SSIM$\uparrow$
  & MAE$\downarrow$ & RMSE$\downarrow$ & SSIM$\uparrow$
  & MAE$\downarrow$ & RMSE$\downarrow$ & SSIM$\uparrow$
  & MAE$\downarrow$ & RMSE$\downarrow$ & SSIM$\uparrow$ \\
\midrule

\multirow{4}{*}{Vel}
& FlatVel-A   & 0.0111 & 0.0180 & 0.9895 & 0.0118 & 0.0178 & 0.9916 & 0.0621 & 0.1233 & 0.9563 & 0.0494 & 0.0839 & 0.8587 & \best{0.0020} & \best{0.0069} & \best{0.9982} \\
& FlatVel-B   & 0.0351 & 0.0876 & 0.9461 & 0.0328 & 0.0787 & 0.9556 & 0.0677 & 0.1493 & 0.8874 & 0.0727 & 0.1457 & 0.8334 & \best{0.0082} & \best{0.0350} & \best{0.9872} \\
& CurveVel-A  & 0.0685 & 0.1202 & 0.8223 & 0.0482 & 0.0976 & 0.8758 & 0.0805 & 0.1411 & 0.8443 & 0.1043 & 0.1592 & 0.7286 & \best{0.0272} & \best{0.0733} & \best{0.9278} \\
& CurveVel-B  & 0.1497 & 0.2801 & 0.6661 & 0.1268 & 0.2611 & 0.7111 & 0.1777 & 0.3179 & 0.6614 & 0.2028 & 0.3141 & 0.5596 & \best{0.0725} & \best{0.1864} & \best{0.8362} \\
\midrule

\multirow{4}{*}{Fault}
& FlatFault-A & 0.0172 & 0.0362 & 0.9798 & 0.0319 & 0.0531 & 0.9798 & 0.0876 & 0.2060 & 0.9340 & 0.0411 & 0.0838 & 0.9184 & \best{0.0059} & \best{0.0180} & \best{0.9939} \\
& FlatFault-B & 0.1055 & 0.1723 & 0.7208 & 0.0925 & 0.1553 & 0.7552 & 0.1416 & 0.2220 & 0.6937 & 0.1346 & 0.1936 & 0.6709 & \best{0.0538} & \best{0.1084} & \best{0.8729} \\
& CurveFault-A& 0.0260 & 0.0602 & 0.9592 & 0.0216 & 0.0505 & 0.9687 & 0.0500 & 0.0966 & 0.9495 & 0.0509 & 0.1013 & 0.8952 & \best{0.0104} & \best{0.0331} & \best{0.9864} \\
& CurveFault-B& 0.1646 & 0.2412 & 0.6163 & 0.1571 & 0.2336 & 0.6033 & 0.3452 & 0.5010 & 0.3941 & 0.1849 & 0.2595 & 0.5729 & \best{0.1062} & \best{0.1838} & \best{0.7314} \\
\midrule

\multirow{2}{*}{Style}
& Style-A     & 0.0610 & 0.0989 & 0.8910 & 0.0612 & 0.1000 & 0.8883 & 0.1429 & 0.2342 & 0.7846 & 0.0848 & 0.1299 & 0.8388 & \best{0.0350} & \best{0.0656} & \best{0.9504} \\
& Style-B     & 0.0586 & 0.0893 & 0.7599 & 0.0649 & 0.0979 & 0.7249 & 0.1702 & 0.2609 & 0.6102 & 0.0693 & 0.1038 & 0.7139 & \best{0.0401} & \best{0.0682} & \best{0.8513} \\
\midrule
\multicolumn{2}{l}{\textbf{Avg.}} 
& 0.0697 & 0.1204 & 0.8351 & 0.0649 & 0.1146 & 0.8454 & 0.1326 & 0.2252 & 0.7716 & 0.0995 & 0.1575 & 0.7590 & \best{0.0361} & \best{0.0779} & \best{0.9136} \\
\bottomrule
\end{tabular}
\end{adjustbox}
\caption{Quantitative comparison on OpenFWI (10 sub-datasets). 
Lower MAE/RMSE and higher SSIM indicate better performance. 
Best results are in \textbf{bold}.}
\label{tab:openfwi_main}
\end{table*}

\subsection{Main Results}
Table~\ref{tab:openfwi_main} summarizes the quantitative comparison of our method against InversionNet, VelocityGAN, UPFWI and FNO on the ten OpenFWI sub-datasets, evaluated using MAE, RMSE and SSIM. For baselines with multiple loss settings reported in OpenFWI, we use the best officially reported results for fair comparison. We follow the official OpenFWI evaluation protocol, using the same metric definitions and code from the OpenFWI repository under the official splits. Our method \textbf{achieves the best performance on all 10/10 sub-datasets}. In terms of averaged metrics, compared with the strongest baseline VelocityGAN, we reduce MAE from 0.0649 to 0.0361 (a \textbf{44.4\%} relative drop) and RMSE from 0.1146 to 0.0779 (a \textbf{32.0\%} relative drop) while improving SSIM from 0.8454 to 0.9136. Compared with the operator baseline FNO, our method further reduces MAE by \textbf{63.7\%} and RMSE by \textbf{50.5\%}. These results indicate that our architecture delivers consistent reconstruction advantages across diverse geological structures and data distributions. A comparison between the subsurface velocity maps predicted by our model and the FNO baseline is shown in Figure~\ref{fig:mainresult}.

\begin{figure}[t]
    \centering
    \includegraphics[width=0.8\linewidth]{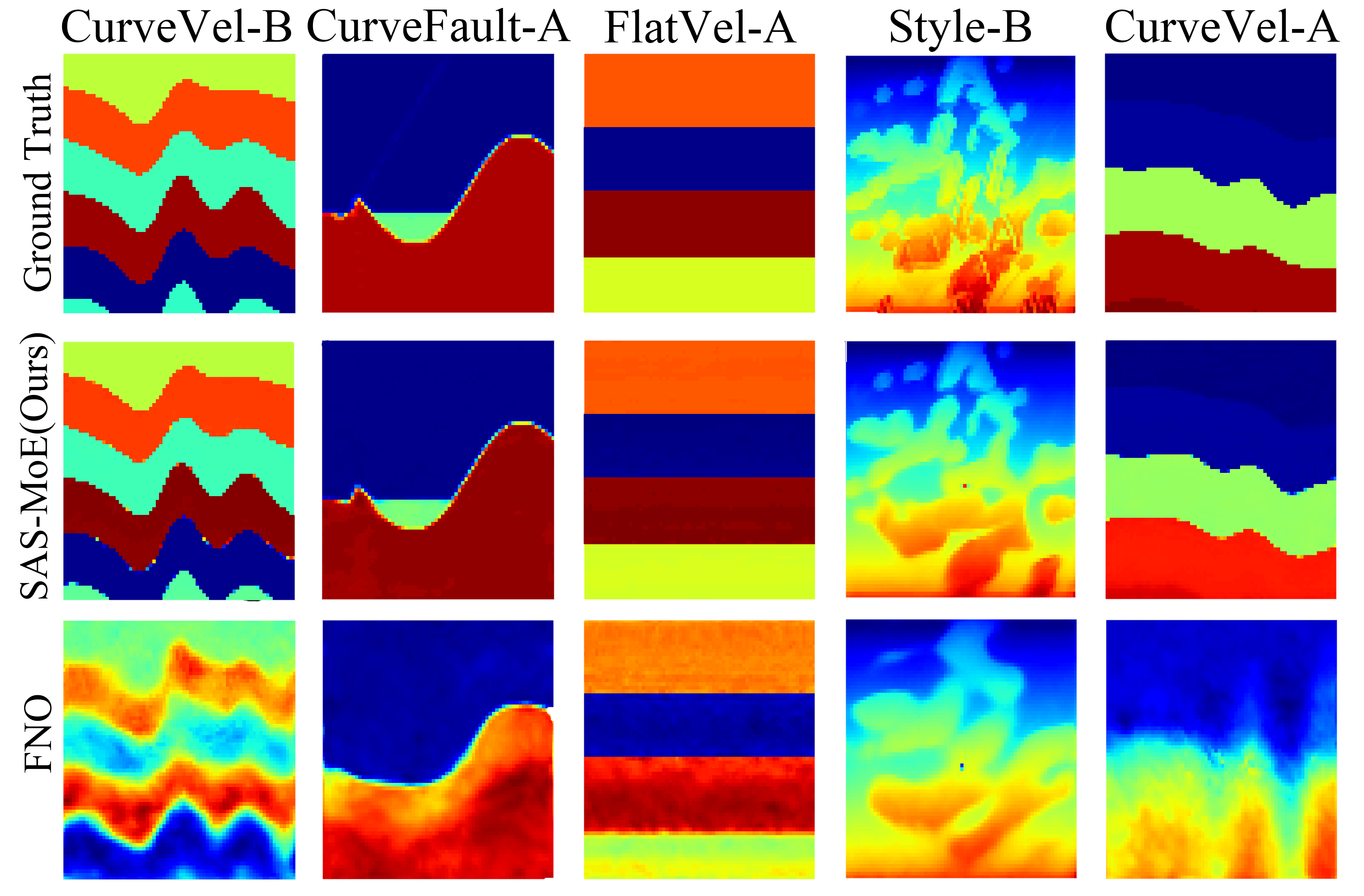}
    \caption{
    \textbf{Comparison of subsurface velocity maps predicted by SPAMoE and the FNO baseline.} 
Top row: ground-truth velocity maps; middle row: SPAMoE predictions; bottom row: FNO predictions.}
    \label{fig:mainresult}
\end{figure}

\noindent\textbf{Discussion:} 
We observe particularly pronounced gains on the more challenging ``B'' sub-datasets, which typically exhibit higher structural complexity and stronger high-frequency reflections. For example, compared with VelocityGAN, our method reduces MAE by 42.8\% on CurveVel-B, 41.8\% on FlatFault-B, and 32.4\% on CurveFault-B. These improvements suggest that explicitly decomposing multi-scale spectral content and modeling it adaptively is crucial for recovering complex details: in such cases, conventional methods that process the entire spectrum through a single pathway are more prone to cross-band interference and over-smoothed reconstructions. In contrast, our Concentric Soft Frequency-Band Decomposition together with the Spectral Energy Attention Router enables band-aware capacity allocation across experts, improving high-frequency detail recovery while preserving the global structure.

On the relatively smoother ``A'' sub-datasets, our method also yields substantial gains (e.g., 83.1\% on FlatVel-A and 81.5\% on FlatFault-A), indicating that the advantage is not limited to highly complex cases. Here, the key benefit lies in better balancing global background accuracy and local-detail sharpness: our frequency decomposition and MoE design can emphasize low-frequency information under smoother structures for stable background reconstruction, while allocating more capacity to higher-frequency bands when finer details are present. For sub-datasets with stronger style perturbations such as Style-B, our method still achieves a clear gain, reducing MAE by 38.2\%. Overall, these results further support that decomposing multi-scale spectral information with adaptive modeling is effective across diverse geological conditions in full-waveform inversion.

\begin{figure}[t]
  \centering
  \begin{subfigure}[b]{0.49\linewidth}
    \centering
    \includegraphics[width=\linewidth]{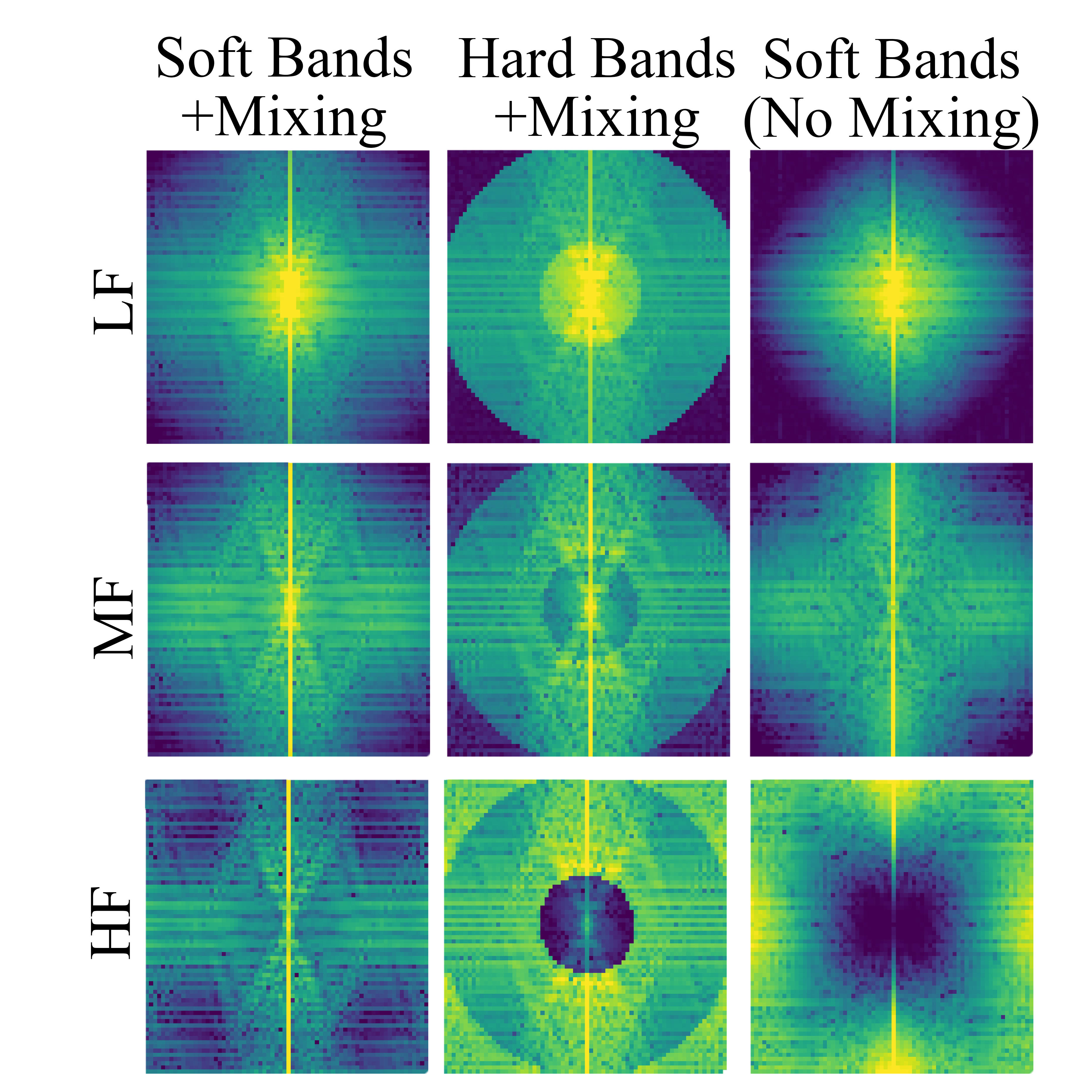}
    \caption{Spectra under different band partition/mixing strategies.}
    \label{fig:ex2}
  \end{subfigure}\hfill
  \begin{subfigure}[b]{0.49\linewidth}
    \centering
    \includegraphics[width=\linewidth]{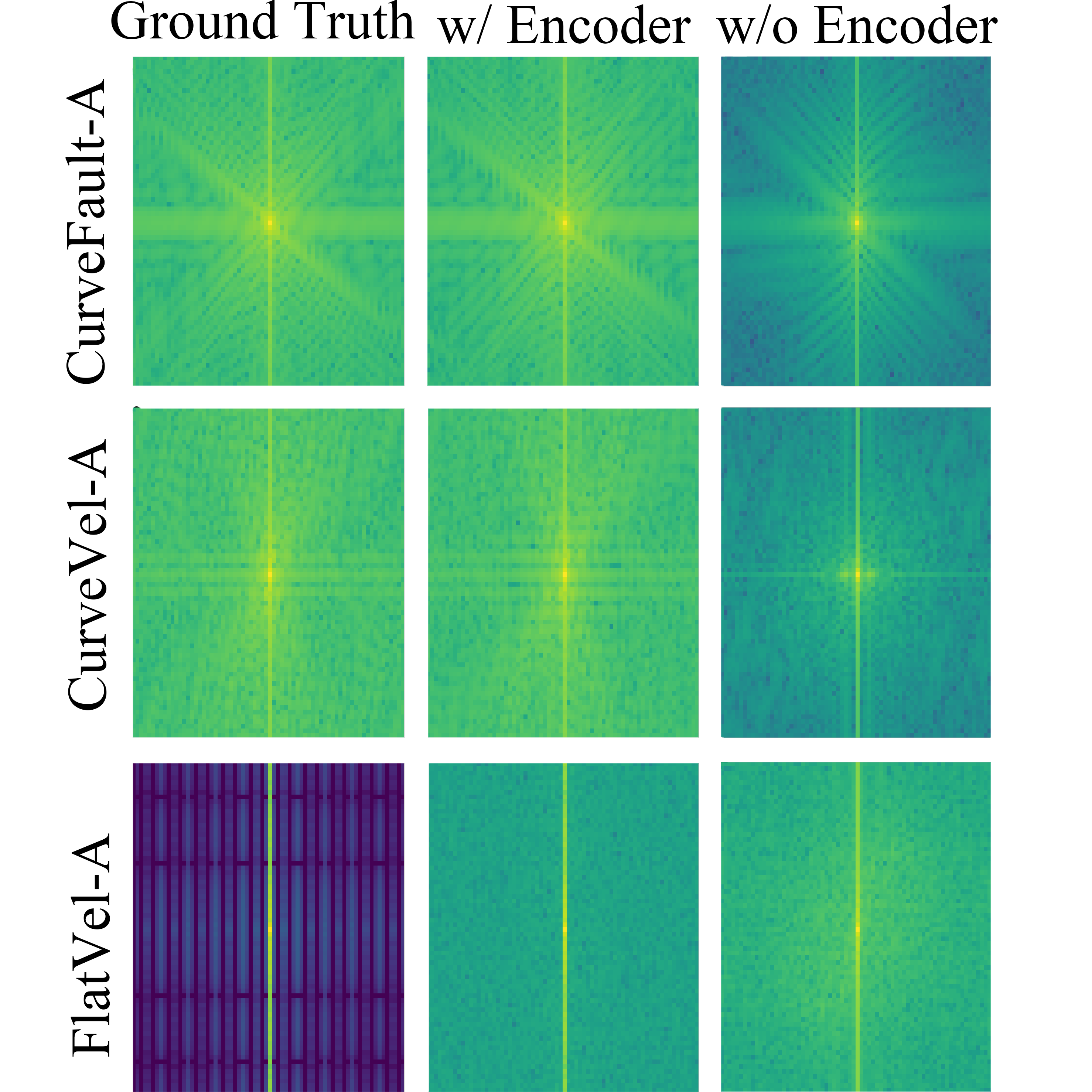}
    \caption{Spectra with and without the encoder.}
    \label{fig:encoder}
  \end{subfigure}
  \caption{Spectral visualizations of ablation study.}
  \label{fig:ex2_encoder}
\end{figure}

\subsection{Ablation Study}
We conduct ablation studies on three representative OpenFWI sub-datasets, CurveFault-A, CurveVel-A, and FlatVel-A, which respectively cover three typical geological regimes: smooth structures, curved stratified layers, and faulted structures. Focusing on the five key components of SPAMoE, we design the following experiments. Unless otherwise stated, we ablate one component at a time while keeping the rest unchanged. Figure~\ref{fig:ex2_encoder} provides spectral visualizations for Ablation Experiments~1 and~4.

\noindent\textbf{1. Effect of Spectral-Preserving DINO Encoder.} With all other components enabled, the mean MAE over the three sub-datasets is 0.0131 with the encoder and 0.0681 without it, showing that the Spectral-Preserving DINO Encoder substantially reduces the overall reconstruction error and serves as a key source of the performance gains. We further compare different encoder designs on CurveVel-B with a fixed FNO backbone (Table~\ref{tab:ablation-encoder}). Compared with directly using the original-resolution input 1000$\times$350 or naively interpolating the waveform to $70{\times}70$, introducing a Spectral-Preserving DINO Encoder markedly improves reconstruction accuracy. Figure~\ref{fig:encoder} shows that incorporating the encoder significantly outperforms the method without it in terms of spectral anisotropy and energy distribution, and the ViT-based variant further achieves the best MAE and RMSE among all configurations. We therefore use the ViT-based DINOv3~\cite{simeoni2025dinov3} encoder as the default choice in subsequent experiments.
\begin{table}[t]
    \centering
    \small
    \setlength{\tabcolsep}{3pt}
    \renewcommand{\arraystretch}{1.15}
    \begin{tabular}{lccc}
        \toprule
        \textbf{Encoder Type} & \textbf{MAE} $\downarrow$ & \textbf{RMSE} $\downarrow$ & \textbf{SSIM} $\uparrow$ \\
        \midrule
        None (original, 1000$\times$350)   & 0.1342 & 0.2579 & 0.7045 \\
        None (resize to 70$\times$70)     & 0.1881 & 0.3165 & 0.6110 \\
        ConvNeXt-based (DINOv3)           & 0.0643 & 0.1713 & 0.8615 \\
        ViT-based (DINOv3)                & \textbf{0.0603} & \textbf{0.1664} & \textbf{0.8626} \\
        \bottomrule
    \end{tabular}
    \caption{Effect of Spectral-Preserving DINO Encoder on 
    CurveVel-B (FNO backbone). ViT-based and ConvNeXt-based encoders are implemented using DINOv3 and trained under identical settings.}
    \label{tab:ablation-encoder}
\end{table}

\noindent\textbf{2. Effect of Spectral Energy Attention Router.} We compare the Spectral Energy Attention Router with a conventional router that relies only on spatial latent features to verify the importance of spectral-energy information for expert selection. Across the three sub-datasets, the Spectral Energy Attention Router achieves an average MAE of 0.01355, while the conventional router attains 0.02218, corresponding to a 38.9\% relative improvement. This is because routing based solely on spatial latent features tends to depend more on the apparent structural patterns of a sample rather than its spectral energy distribution, whereas multi-scale coupling across structures is a key difficulty in FWI, making conventional routing more prone to suboptimal expert assignment on complex samples. In contrast, the Spectral Energy Attention Router takes the \emph{Amplitude Energy Map} as input and aggregates global energy-distribution patterns via attention, so that routing decisions are explicitly driven by spectral structure, leading to more stable expert selection and improved inversion accuracy.

\noindent\textbf{3. Effect of Concentric Soft Frequency-Band Decomposition.}
Using concentric soft frequency-band decomposition (Gaussian concentric soft masks) significantly outperforms hard frequency-band partitioning (Step-function masks), reducing the MAE averaged over the three sub-datasets from 0.01721 to 0.01355 (a 21.3\% relative reduction). The core reason is that the indicator function used in hard partitioning introduces non-differentiable step boundaries, causing abrupt truncation in the frequency domain and consequent information loss around band cutoffs, so that subsequent routing and expert modeling cannot leverage frequency-structure information outside the truncated boundaries. In contrast, soft-band decomposition adopts continuous and differentiable Gaussian radial masks, providing smooth transitions and mild overlap between adjacent bands; this stabilizes the band separation process and preserves informative frequency components near band boundaries, thereby better supporting reliable band-wise modeling and expert collaboration.This difference is also evident in the first two columns of Fig.~\ref{fig:ex2}(a): under soft bands, the LF/MF/HF energy distributions exhibit smooth radial decay with natural transitions across bands, whereas hard bands produce sharper cutoffs and clear ring-shaped boundary artifacts, particularly in mid/high-frequency bands where unnatural energy discontinuities are more pronounced. These spectral-level observations are consistent with the quantitative results, further supporting that differentiable soft-band decomposition is more suitable for FWI with multi-scale spectral characteristics.


\noindent\textbf{4. Effect of Adaptive Frequency-Preference Mechanism.}
Compared with the full model, removing the Adaptive adaptive frequency-preference mechanism increases the MAE averaged over the three sub-datasets from 0.01355 to 0.01645 (a 17.6\% relative increase), indicating that this mechanism is necessary for stably exploiting the benefits of multiple experts. This is because, in FWI, the spectral energy of velocity models is rarely confined to a single fixed band; in particular, around complex structures such as fault boundaries and curved stratigraphy, informative energy often spans neighboring bands with continuous transitions. Under a static assignment that strictly binds each expert to one band, experts become constrained by band boundaries: they cannot access transitional frequency components from other bands, and residual non-target-band energy may also interfere with specialization, ultimately leading to insufficient division of labor or expert mismatch.In contrast, the Adaptive Frequency-Preference Mechanism allows each expert, while being assigned its preferred band, to softly absorb necessary transitional components from adjacent bands. This better matches the continuous multi-scale spectral variation in FWI, enabling experts to model their preferred frequency patterns more stably and collaborate more effectively. The band-wise spectra visualized in the first and last columns of Fig.~\ref{fig:ex2} further support this: with the mechanism enabled, the energy distributions are more continuous and cross-band transitions are more natural; without it, band energy becomes more fragmented with noticeable missing components, especially in mid/high-frequency bands. This indicates that static band binding weakens the utilization of cross-band structural information and thus degrades overall inversion accuracy.

\begin{table}[t]
\centering
    \small
    \setlength{\tabcolsep}{1.5pt} 
    \renewcommand{\arraystretch}{1.15}
    \begin{tabular}{lccc}
        \toprule
        \textbf{Model} & \textbf{CurveFault-A} & \textbf{FlatVel-A} & \textbf{CurveVel-A} \\
        \midrule
        FNO & 0.01130 & 0.00242 & 0.02704\\
        MNO & 0.01158 & 0.00225 & 0.02897\\
        LNO & 0.01122 & 0.00223 & 0.02752 \\
        \textbf{MoE (Ours)}& \textbf{0.01107} & \textbf{0.00186} & \textbf{0.02664}\\
        \bottomrule
    \end{tabular}
    \caption{Performance comparison of the proposed Adaptive Spectral MoE against single expert operators (FNO, MNO, and LNO) in terms of MAE.}
    \label{tab:ablation-moe-comparison}
\end{table}


\noindent\textbf{5. Effect of Multi-Operator MoE vs. Single-Operator Baselines.}
We compare the multi-operator MoE with single-operator variants (FNO/MNO/LNO) on the three ablation sub-datasets (Table~\ref{tab:ablation-moe-comparison}). The MoE achieves the lowest MAE on CurveFault-A, FlatVel-A, and CurveVel-A, indicating that combining multiple operator experts can consistently outperform any individual operator. The underlying reason is that different neural operators exhibit distinct modeling preferences over frequency content and spatial structures: some are better at capturing smooth backgrounds and low-frequency trends, while others are more sensitive to local variations and higher-frequency details. A single-operator model must cover all scale components with a single inductive bias, which often leads to a trade-off between global structure and fine details; in contrast, the multi-operator MoE enables adaptive collaboration among experts conditioned on each sample, thereby preserving the overall structure while more effectively complementing local details and variations.

Numerically, the MoE yields the most evident gain on FlatVel-A, while improvements on CurveVel-A and CurveFault-A are more moderate yet remain consistently best. When the sample structure is more complex or the band-wise distribution is more imbalanced, multi-expert collaboration tends to be more beneficial; when a single operator already captures the dominant structure well, the MoE can still provide stable marginal gains through expert complementarity.

\subsection{Additional Experiments on Pipe Flows}
\label{supply:pipe}
\subsubsection{Experiment Introduction}
This experiment aims to verify the generalization capability of our proposed model in addressing fluid dynamics problems, specifically its ability to solve PDEs under irregular geometric boundaries. We selected the classic pipe flows problem as our testing benchmark. Although the pipe flows problem is primarily governed by the Navier-Stokes equations-differing in physical mechanism from the wave equation involved in FWI—both share significant mathematical commonalities: they rely on capturing features within specific frequency domains of complex physical fields and are highly sensitive to variations in boundary conditions. By conducting experiments on the pipe flows dataset, we aim to demonstrate that our model not only performs excellently in FWI tasks but also possesses the potential to analyze and process other types of PDE problems with complex spectral characteristics, enabling precise modeling of fluid dynamic behaviors within confined spaces.

\subsubsection{Dataset Introduction}
The dataset used in this experiment is derived from the standard benchmark data generated in the Geo-FNO~\cite{li2023geofno} paper, which primarily simulates steady-state fluid flow within a two-dimensional pipe. A distinguishing feature of this dataset is the random variation in the pipe's geometry, which poses a challenge for the model in learning the nonlinear relationship between grid mapping and physical fields.

\noindent\textbf{Data Composition.} We divided the complete Pipe dataset into a training set of 1,000 samples and a test set of 200 samples. Each sample consists of input data and target output data. The input data comprises the geometric coordinates of the grid points, $X \in \mathbb{R}^{129\times129}$ and $Y \in \mathbb{R}^{129\times129}$, while the target output data represents the fluid velocity field. This experiment focuses primarily on the horizontal velocity component, $Q \in \mathbb{R}^{129\times129}$.

\noindent\textbf{Preprocessing.} For the input data, we first applied a log transform followed by min–max normalization. For the output data, we applied min–max normalization as preprocessing.

\subsubsection{Experimental Setup}
To fairly evaluate model performance, we adopted a rigorous comparative experimental setup. We selected the LaMO~\cite{tiwari2025latent} model, which has demonstrated superiority in handling complex geometric PDE problems, as our primary baseline. Both our model and LaMO were trained and tested on the exact same dataset generated by Geo-FNO, using the same split ratio (1,000 samples for training and 200 samples for testing). Furthermore, we adopted the same evaluation metric as LaMO—the $Relative\,\ell_2$ Error—as our core metric, The evaluation was conducted on the test set, and the results were averaged. This metric eliminates dimensional influence and objectively reflects the overall degree of deviation between the predicted physical field and the ground truth. The calculation formula is as follows:
\begin{equation} 
\mathrm{Relative}\,\ell_2 = \frac{\| \hat{y} - y \|_2}{\| y \|_2}
\end{equation}
where $\hat{y}$ represents the velocity field predicted by the model, and $y$ represents the ground truth velocity field.

\subsubsection{Result}
Experimental results show that the LaMO model achieves a relative $\ell_2$ error of 0.0038 on this dataset, whereas our method further reduces the error to 0.0025, corresponding to an improvement of approximately 34.2\%. As visualized in Figure~\ref{fig:pipe}, our predicted flow fields exhibit close agreement with the ground-truth pipe-flow turbulence, indicating higher fidelity in capturing fine-scale structures of the velocity distribution. These results provide strong evidence for the effectiveness of our architectural design: even when transferred from seismic inversion to fluid dynamics, the proposed model maintains robust fitting capability. Beyond validating its spectral modeling advantages consistent with those observed in FWI, this also highlights its potential as a general-purpose PDE solver for a broader range of scientific computing tasks.

\begin{figure}[h]
    \centering
    \includegraphics[width=0.8\linewidth]{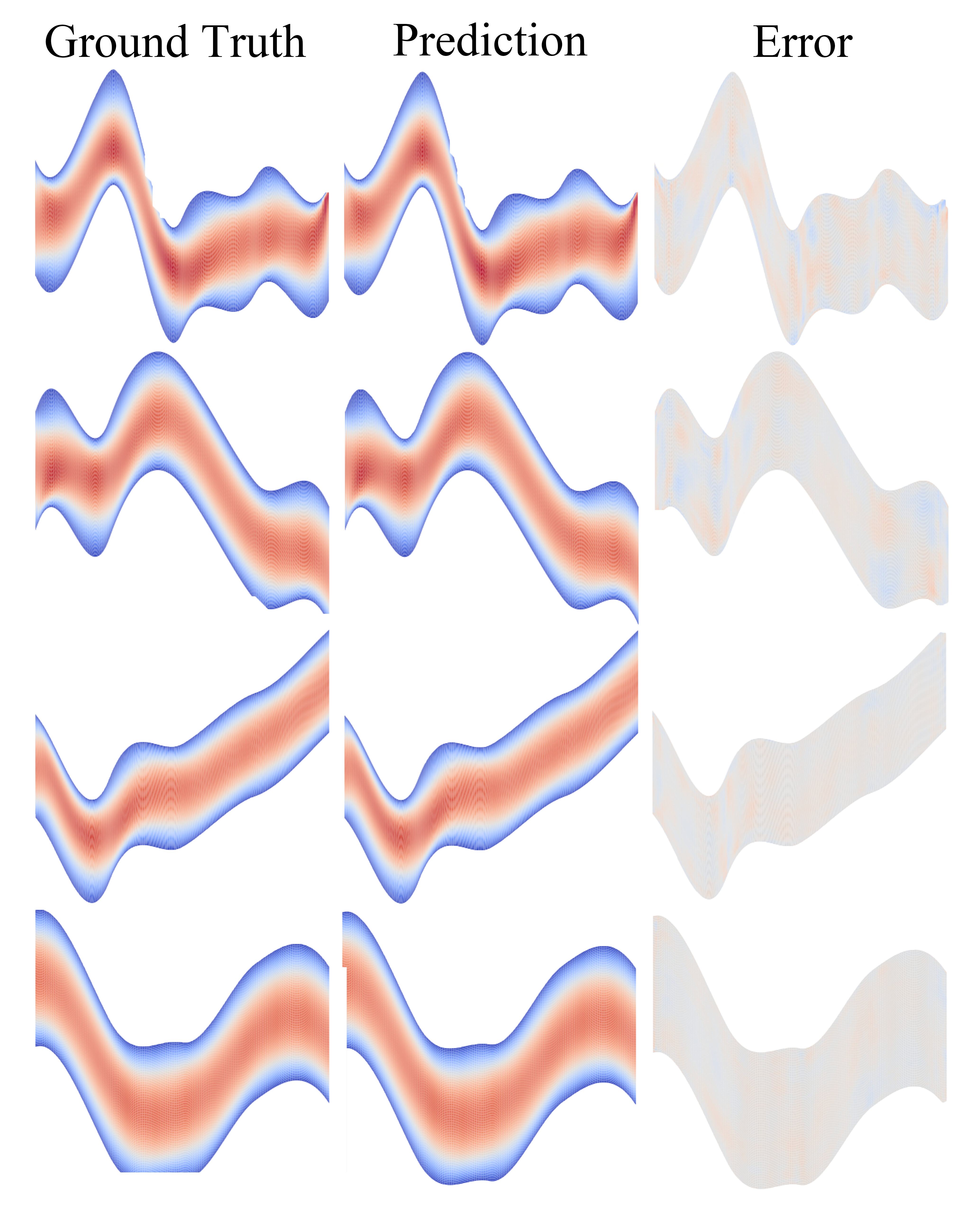}
    \caption{
        \textbf{Qualitative comparison of pipe flow velocity fields on the Pipe Flows dataset.} From left to right, each column shows the ground-truth velocity field, the prediction of our model, and the corresponding absolute error.Each row corresponds to a different test sample.
    }
    \label{fig:pipe}
\end{figure}

\section{Conclusion}
In this paper, we proposed SPAMoE, a unified framework for full-waveform inversion. 
By integrating a Spectral-Preserving DINO Encoder and an Adaptive Spectral Mixture-of-Experts module, SPAMoE effectively addresses the challenges of multi-scale frequency entanglement and the subsequent modeling of disentangled components. Experiments show that SPAMoE reduces the average MAE over the ten OpenFWI sub-datasets from 0.0649 (the best reported baseline) to 0.0361, a \textbf{44.4\%} relative reduction. In addition, SPAMoE also achieves strong performance on the pipe flows task, suggesting that the proposed framework has the potential to generalize to other challenging PDE learning problems.

\appendix

\section{Detailed Definition of Notation}

\noindent\textbf{A1. Centered Representation of the 2D Discrete Fourier Transform.}
\label{app:freq_def}
As described in section~\ref{subsec:preliminaries} of the main text, we define the operator $\mathrm{shift}(\cdot)$ to move the zero-frequency component to the center of the spectrum. 
Let the two-dimensional discrete fourier transform be denoted by $\widehat{v} = \mathcal{F}(u) \in \mathbb{C}^{H \times W}$. 
The centered spectral representation is then defined as
$
\widehat{u} = \mathcal{U}(u) = \mathrm{shift}(\widehat{v}) = \mathrm{shift}(\mathcal{F}(u))
$, and we define the inverse transform as 
$
u = \mathcal{U}^{-1}(\widehat{u}).
$

\noindent\textbf{A2. Definition of the Frequency Coordinate Grid.}
\label{app:grid_def}
As described in section~\ref{subsec:preliminaries}, we define the normalized radial frequency $r(\omega)$, where $\omega=(i,j)$ denotes the discrete frequency index in the centered spectrum; that is, $r(\omega)=r(i,j)$.

The normalized radial frequency $r(i,j)$ is defined on the centered spectrum as
\begin{equation}
r(i,j)
=
\frac{
\sqrt{
\left(-1 + \frac{2j}{W-1}\right)^2
+
\left(-1 + \frac{2i}{H-1}\right)^2
}
}{
\max\limits_{p,q}
\sqrt{
\left(-1 + \frac{2q}{W-1}\right)^2
+
\left(-1 + \frac{2p}{H-1}\right)^2
}
},
\end{equation}
where $H$ and $W$ denote the height and width of the spatial grid, respectively.

This formulation maps the discrete frequency coordinates to a normalized Cartesian domain $[-1,1]\times[-1,1]$, with the spectral center corresponding to zero frequency.
The resulting quantity $r(i,j)\in[0,1]$ therefore provides a monotonic measure of the physical frequency magnitude with respect to the radial distance from the spectrum center.

Based on this radial metric, the frequency domain can be naturally decomposed into concentric bands by thresholding $r(i,j)$, which enables frequency-aware partitioning and subsequent band-wise processing in our spectral routing module.

\noindent\textbf{A3. Definition of Spectral Energy.}
\label{app:energy_def}
As described in section~\ref{subsec:preliminaries} of the main text, the power spectrum is defined as $P_u(\omega) = |\widehat{u}(\omega)|^2$. 
Given two frequency sets $\Omega_L$ and $\Omega_H$, the spectral energy of the field $u$ in the low-frequency and high-frequency bands is defined as
$
    E_L(u) = \sum_{\omega\in\Omega_L} P_u(\omega),
    \,
    E_H(u) = \sum_{\omega\in\Omega_H} P_u(\omega).
$
The high-to-low frequency energy ratio is further defined as
\begin{equation}
    \mathrm{HL}(u)
    =
    \frac{E_H(u)}{E_L(u) + \varepsilon},
\end{equation}
where $\varepsilon > 0$ is a numerical stability term introduced to avoid division by zero.

\bibliographystyle{named}
\bibliography{ijcai26}

@article{deng2022openfwi,
  title={OpenFWI: Large-scale multi-structural benchmark datasets for full waveform inversion},
  author={Deng, Chengyuan and Feng, Shihang and Wang, Hanchen and Zhang, Xitong and Jin, Peng and Feng, Yinan and Zeng, Qili and Chen, Yinpeng and Lin, Youzuo},
  journal={Advances in Neural Information Processing Systems},
  volume={35},
  pages={6007--6020},
  year={2022}
}

@article{geng2018frequency,
  title={Frequency-domain full-waveform inversion with non-linear descent directions},
  author={Geng, Yu and Pan, Wenyong and Innanen, Kristopher A},
  journal={Geophysical Journal International},
  volume={213},
  number={2},
  pages={739--756},
  year={2018},
  publisher={Oxford University Press}
}

@article{hu2023augmented,
  title={Augmented Physics-Informed Neural Networks (APINNs): A gating network-based soft domain decomposition methodology},
  author={Hu, Zheyuan and Jagtap, Ameya D and Karniadakis, George Em and Kawaguchi, Kenji},
  journal={Engineering Applications of Artificial Intelligence},
  volume={126},
  pages={107183},
  year={2023},
  publisher={Elsevier}
}

@article{lutjens2022multiscale,
  title={Multiscale neural operator: Learning fast and grid-independent pde solvers},
  author={L{\"u}tjens, Bj{\"o}rn and Crawford, Catherine H and Watson, Campbell D and Hill, Christopher and Newman, Dava},
  journal={arXiv preprint arXiv:2207.11417},
  year={2022}
}

@article{rahman2022u,
  title={U-no: U-shaped neural operators},
  author={Rahman, Md Ashiqur and Ross, Zachary E and Azizzadenesheli, Kamyar},
  journal={arXiv preprint arXiv:2204.11127},
  year={2022}
}

@article{shazeer2017outrageously,
  title={Outrageously large neural networks: The sparsely-gated mixture-of-experts layer},
  author={Shazeer, Noam and Mirhoseini, Azalia and Maziarz, Krzysztof and Davis, Andy and Le, Quoc and Hinton, Geoffrey and Dean, Jeff},
  journal={arXiv preprint arXiv:1701.06538},
  year={2017}
}

@article{simeoni2025dinov3,
  title={Dinov3},
  author={Sim{\'e}oni, Oriane and Vo, Huy V and Seitzer, Maximilian and Baldassarre, Federico and Oquab, Maxime and Jose, Cijo and Khalidov, Vasil and Szafraniec, Marc and Yi, Seungeun and Ramamonjisoa, Micha{\"e}l and others},
  journal={arXiv preprint arXiv:2508.10104},
  year={2025}
}

@article{virieux2009overview,
  title={An overview of full-waveform inversion in exploration geophysics},
  author={Virieux, Jean and Operto, St{\'e}phane},
  journal={Geophysics},
  volume={74},
  number={6},
  pages={WCC1--WCC26},
  year={2009},
  publisher={Society of Exploration Geophysicists}
}

@article{wu2019inversionnet,
  title={InversionNet: An efficient and accurate data-driven full waveform inversion},
  author={Wu, Yue and Lin, Youzuo},
  journal={IEEE Transactions on Computational Imaging},
  volume={6},
  pages={419--433},
  year={2019},
  publisher={IEEE}
}

@article{yang2023rapid,
  title={Rapid seismic waveform modeling and inversion with neural operators},
  author={Yang, Yan and Gao, Angela F and Azizzadenesheli, Kamyar and Clayton, Robert W and Ross, Zachary E},
  journal={IEEE Transactions on Geoscience and Remote Sensing},
  volume={61},
  pages={1--12},
  year={2023}
}

@inproceedings{zhang2019velocitygan,
  title={VelocityGAN: Subsurface velocity image estimation using conditional adversarial networks},
  author={Zhang, Zhongping and Wu, Yue and Zhou, Zheng and Lin, Youzuo},
  booktitle={2019 IEEE Winter Conference on Applications of Computer Vision (WACV)},
  pages={705--714},
  year={2019},
  organization={IEEE}
}

@article{deighan2025mnoe,
  title   = {Mixture of Neural Operator Experts for Learning Boundary Conditions and Model Selection},
  author  = {Deighan, Dwyer and Actor, Jonas A. and Patel, Ravi G. and others},
  journal = {arXiv preprint arXiv:2502.04562},
  year    = {2025}
}

@article{chen2025freqmoe,
  title={FreqMoE: Dynamic Frequency Enhancement for Neural PDE Solvers},
  author={Chen, Tianyu and Zhou, Haoyi and Li, Ying and Wang, Hao and Zhang, Zhenzhe and Zhu, Tianchen and Zhang, Shanghang and Li, Jianxin},
  journal={arXiv preprint arXiv:2505.06858},
  year={2025}
}

@article{partyka1999interpretational,
  title={Interpretational applications of spectral decomposition in reservoir characterization},
  author={Partyka, Greg and Gridley, James and Lopez, John},
  journal={The leading edge},
  volume={18},
  number={3},
  pages={353--360},
  year={1999},
  publisher={Society of Exploration Geophysicists}
}

@article{jin2021unsupervised,
  title={Unsupervised learning of full-waveform inversion: Connecting CNN and partial differential equation in a loop},
  author={Jin, Peng and Zhang, Xitong and Chen, Yinpeng and Huang, Sharon Xiaolei and Liu, Zicheng and Lin, Youzuo},
  journal={arXiv preprint arXiv:2110.07584},
  year={2021}
}

@article{lifourier,
  title={Fourier neural operator for parametric partial differential equations},
  author={Li, Zongyi and Kovachki, Nikola and Azizzadenesheli, Kamyar and Liu, Burigede and Bhattacharya, Kaushik and Stuart, Andrew and Anandkumar, Anima},
  journal={arXiv preprint arXiv:2010.08895},
  year={2020}
}

@article{li2024local,
  title={Local neural operator for solving transient partial differential equations on varied domains},
  author={Li, Hongyu and Ye, Ximeng and Jiang, Peng and Qin, Guoliang and Wang, Tiejun},
  journal={Computer Methods in Applied Mechanics and Engineering},
  volume={427},
  pages={117062},
  year={2024},
  publisher={Elsevier}
}

@article{kovachki2023neural,
  title={Neural operator: Learning maps between function spaces with applications to pdes},
  author={Kovachki, Nikola and Li, Zongyi and Liu, Burigede and Azizzadenesheli, Kamyar and Bhattacharya, Kaushik and Stuart, Andrew and Anandkumar, Anima},
  journal={Journal of Machine Learning Research},
  volume={24},
  number={89},
  pages={1--97},
  year={2023}
}

@inproceedings{bischof2022mixture,
  title={Mixture-of-experts-ensemble meta-learning for physics-informed neural networks},
  author={Bischof, Rafael and Kraus, Michael A},
  booktitle={Proceedings of 33. forum bauinformatik},
  year={2022}
}

@article{sharma2024ensemble,
  title={Ensemble and Mixture-of-Experts DeepONets For Operator Learning},
  author={Sharma, Ramansh and Shankar, Varun},
  journal={arXiv preprint arXiv:2405.11907},
  year={2024}
}

@article{wen2022ufno,
  title={U-FNO—An enhanced Fourier neural operator-based deep-learning model for multiphase flow},
  author={Wen, Gege and Li, Zongyi and Azizzadenesheli, Kamyar and Anandkumar, Anima and Benson, Sally M},
  journal={Advances in Water Resources},
  volume={163},
  pages={104180},
  year={2022},
  publisher={Elsevier}
}

@article{seidman2023variational,
  title={Variational autoencoding neural operators},
  author={Seidman, Jacob H and Kissas, Georgios and Pappas, George J and Perdikaris, Paris},
  journal={arXiv preprint arXiv:2302.10351},
  year={2023}
}

@article{viknesh2025differentiable,
  title={Differentiable Autoencoding Neural Operator for Interpretable and Integrable Latent Space Modeling},
  author={Viknesh, Siva and Arzani, Amirhossein},
  journal={arXiv preprint arXiv:2510.00233},
  year={2025}
}

@article{shi2025mesh,
  title={Mesh-Informed Neural Operator: A Transformer Generative Approach},
  author={Shi, Yaozhong and Ross, Zachary E and Asimaki, Domniki and Azizzadenesheli, Kamyar},
  journal={arXiv preprint arXiv:2506.16656},
  year={2025}
}

@article{zou2025ambient,
  title={Ambient noise full waveform inversion with neural operators},
  author={Zou, Caifeng and Ross, Zachary E and Clayton, Robert W and Lin, Fan-Chi and Azizzadenesheli, Kamyar},
  journal={Journal of Geophysical Research: Solid Earth},
  volume={130},
  number={11},
  pages={e2025JB031624},
  year={2025},
  publisher={Wiley Online Library}
}

@article{zhu2023fourier,
  title={Fourier-DeepONet: Fourier-enhanced deep operator networks for full waveform inversion with improved accuracy, generalizability, and robustness},
  author={Zhu, Min and Feng, Shihang and Lin, Youzuo and Lu, Lu},
  journal={Computer Methods in Applied Mechanics and Engineering},
  volume={416},
  pages={116300},
  year={2023},
  publisher={Elsevier}
}

@article{li2023geofno,
  title={Fourier neural operator with learned deformations for pdes on general geometries},
  author={Li, Zongyi and Huang, Daniel Zhengyu and Liu, Burigede and Anandkumar, Anima},
  journal={Journal of Machine Learning Research},
  volume={24},
  number={388},
  pages={1--26},
  year={2023}
}

@article{tiwarilatent,
  title={Latent Mamba Operator for Partial Differential Equations},
  author={Tiwari, Karn and Dutta, Niladri and Krishnan, NM and others},
  journal={arXiv preprint arXiv:2505.19105},
  year={2025}
}

@article{wang2004ssim,
  title={Image quality assessment: from error visibility to structural similarity},
  author={Wang, Zhou and Bovik, Alan C and Sheikh, Hamid R and Simoncelli, Eero P},
  journal={IEEE transactions on image processing},
  volume={13},
  number={4},
  pages={600--612},
  year={2004},
  publisher={IEEE}
}

@article{tiwari2025latent,
  title={Latent mamba operator for partial differential equations},
  author={Tiwari, Karn and Dutta, Niladri and Krishnan, NM and others},
  journal={arXiv preprint arXiv:2505.19105},
  year={2025}
}

\clearpage
\appendix

\begin{table*}[htpb]
\centering
\small
\setlength{\tabcolsep}{4pt}
\renewcommand{\arraystretch}{1.10}
\begin{adjustbox}{max width=\linewidth}
\begin{tabular}{p{3.0cm} p{8.4cm}}
\toprule
\textbf{Symbol} & \textbf{Meaning} \\
\midrule

$x\in\mathbb{R}^{N_s\times T\times N_r}$ 
& Seismic observations with $N_s$ shots, $T$ temporal samples, and $N_r$ receivers. \\

$x'\in\mathbb{R}^{T\times (N_sN_r)}$ 
& Reshaped panoramic observation. \\

$y\in\mathbb{R}^{H\times W}$ 
& Ground-truth subsurface velocity model. \\

$\hat{y}\in\mathbb{R}^{H\times W}$ 
& Predicted subsurface velocity model. \\

$E_\theta$ 
& Spectral-Preserving DINO Encoder. \\

$z\in\mathbb{R}^{C\times H\times W}$ 
& Spatially aligned latent representation output by the encoder. \\

$R$ 
& Linear readout operator projecting latent features to a comparable spatial field. \\

$u_c\in\mathbb{R}^{H\times W}$ 
& Comparable spatial field used for spectral analysis and theoretical derivations. \\

$F$ 
& Downstream prediction operator (e.g., FNO). \\

$\hat{y}_c=F(u_c)$ 
& Downstream prediction expressed in the theoretical analysis. \\

\midrule
$\omega=(i,j)$ 
& Discrete frequency index in the centered 2D Fourier domain, where $(i,j)$ denotes the frequency coordinate after zero-frequency shifting. \\

$\mathcal{U}(\cdot)$ 
& Centered 2D discrete Fourier transform (DFT) operator. \\

$\widehat{u}=\mathcal{U}(u)\in\mathbb{C}^{H\times W}$ 
& Centered spectral representation of a spatial field $u$. \\

$P_u(\omega)$ 
& Power spectrum at frequency index $\omega$, defined as $|\widehat{u}(\omega)|^2$. \\

$r(\omega)=r(i,j)\in[0,1]$ 
& Normalized radial frequency associated with frequency index $\omega=(i,j)$. \\

$r_{\text{split}}$ 
& Threshold separating low- and high-frequency regions. \\

$\Omega_L$ 
& Low-frequency index set $\{\omega \mid r(\omega)<r_{\text{split}}\}$. \\

$\Omega_H$ 
& High-frequency index set $\{\omega \mid r(\omega)\ge r_{\text{split}}\}$. \\

$E_L(u)$ 
& Low-frequency spectral energy of field $u$. \\

$E_H(u)$ 
& High-frequency spectral energy of field $u$. \\

$\mathrm{HL}(u)$ 
& High-to-low frequency energy ratio of $u$. \\

\midrule
$K$ 
& Number of concentric soft frequency bands. \\

$c_k$ 
& Normalized center of the $k$-th frequency band. \\

$M_k(i,j)$ 
& Gaussian soft mask applied to the $k$-th frequency band at location $(i,j)$. \\

$\gamma$ 
& Band sharpness parameter. \\

$z_k\in\mathbb{R}^{C\times H\times W}$ 
& Band-wise latent feature reconstructed from the $k$-th frequency band. \\

$f_e\in[0,1]$ 
& Learnable frequency-preference parameter of expert $e$. \\

$\eta$ 
& Frequency-affinity sharpness. \\

$\pi_{e,k}$ 
& Mixing weight of frequency band $k$ for expert $e$. \\

$\tilde{z}_e$ 
& Input feature to expert $e$ after adaptive frequency-preference mixing. \\

$g\in\mathbb{R}^{N_E}$ 
& Router logits. \\

$\mathcal{T}(x)$ 
& Index set of Top-$E$ experts selected by the router for sample $x$. \\

$\alpha_e(x)$ 
& Normalized gating weight of expert $e$ for sample $x$. \\

$\mathcal{O}_e$ 
& Operator implemented by expert $e$. \\

\midrule
$\delta$ 
& High-frequency non-contractiveness constant. \\

$\kappa$ 
& Low-frequency controllability constant. \\

$m$ 
& Lower bound of downstream operator amplification. \\

$M$ 
& Upper bound of downstream operator amplification. \\

$I$ 
& Interpolation operator used as a baseline frontend. \\

$H_I(\omega)$ 
& Frequency response of interpolation operator $I$. \\

$\alpha$ 
& Upper bound of $|H_I(\omega)|$ over the high-frequency region $\Omega_H$. \\

$\beta$ 
& Lower bound of $|H_I(\omega)|$ over the low-frequency region $\Omega_L$. \\

\bottomrule
\end{tabular}
\end{adjustbox}
\caption{Comprehensive notation list for the main method and theoretical analysis.}
\label{tab:notation_com}
\end{table*}

\section{Table of Notions}
\label{app:notation_compact}
Table~\ref{tab:notation_com} provides a comprehensive list of notations used in the main method and theoretical analysis.

\section{Theoretical Proofs}
\label{supply:proofs}
 
This supplementary material provides the formal proofs for the spectral properties of interpolation and the Spectral-Preserving DINO Encoder presented in the main text.
 
\subsection{Proof of Theorem~\ref{thm:interp_lowpass}}
\label{supply:proof_interp}

This subsection provides the detailed proof of Theorem~\ref{thm:interp_lowpass}, demonstrating that the spatial interpolation operator imposes an upper bound on the high-to-low frequency energy ratio (HL) when high-frequency components are attenuated more strongly than low-frequency components.
 
\begin{proof}
By the definition of the power spectrum and the multiplicative frequency-domain response assumption in Eq.~\eqref{eq:response}, we have:
\begin{align}
    P_{I(u)}(\omega)
    &= |\widehat{I(u)}(\omega)|^2 \nonumber \\
    &= |H_I(\omega)\widehat u(\omega)|^2 \nonumber \\
    &= |H_I(\omega)|^2\,P_u(\omega).
\end{align}
Therefore, the high-frequency energy satisfies:
\begin{align}
    E_H(I(u))
    &= \sum_{\omega\in\Omega_H}P_{I(u)}(\omega) \nonumber \\
    &= \sum_{\omega\in\Omega_H}|H_I(\omega)|^2P_u(\omega) \nonumber \\
    &\le \alpha^2\sum_{\omega\in\Omega_H}P_u(\omega) \nonumber \\
    &= \alpha^2E_H(u).
\end{align}
Similarly, the low-frequency energy satisfies:
\begin{align}
    E_L(I(u))
    &= \sum_{\omega\in\Omega_L}|H_I(\omega)|^2P_u(\omega) \nonumber \\
    &\ge \beta^2\sum_{\omega\in\Omega_L}P_u(\omega) \nonumber \\
    &= \beta^2E_L(u).
\end{align}
Substituting these bounds into the definition of the HL ratio yields:
\begin{align}
    \mathrm{HL}(I(u)) 
    &= \frac{E_H(I(u))}{E_L(I(u))+\varepsilon} \nonumber \\
    &\le \frac{\alpha^2E_H(u)}{\beta^2E_L(u)+\varepsilon} \nonumber \\
    &\le \frac{\alpha^2}{\beta^2}\cdot\frac{E_H(u)}{E_L(u)+\varepsilon} \nonumber \\
    &= \frac{\alpha^2}{\beta^2}\mathrm{HL}(u),
\end{align}
where the second inequality follows from $\beta\le 1$. Since $0<\alpha<\beta$, we have $\alpha^2/\beta^2<1$, indicating that the HL ratio is strictly reduced. This proves the upper bound constraint imposed by the interpolation operator.
\end{proof}
 
\subsection{Proof of Theorem~\ref{main:thm}}
\label{supply:proof_dino}

This subsection presents the formal proof of Theorem~\ref{main:thm}, establishing the spectral preservation property of the proposed Spectral-Preserving DINO Encoder.
 
\begin{proof}
From assumption A3~\eqref{eq:bounded_resp} regarding the boundedness of the downstream operator, the high-frequency energy of the final prediction satisfies:
\begin{equation}
    E_H(\hat y_c) = E_H(F(u_c)) \ge mE_H(u_c).
\end{equation}
Combining this with the high-frequency non-contraction assumption A1~\eqref{eq:hf_noncontract}, we obtain:
\begin{equation}
    E_H(\hat y_c) \ge m\cdot\delta\,E_H(y).
\end{equation}
Similarly, from assumption A3~\eqref{eq:bounded_resp}, the low-frequency energy satisfies:
\begin{equation}
    E_L(\hat y_c) = E_L(F(u_c)) \le ME_L(u_c).
\end{equation}
Substituting these bounds into the definition of the HL ratio yields:
\begin{align}
    \mathrm{HL}(\hat y_c) 
    &= \frac{E_H(\hat y_c)}{E_L(\hat y_c)+\varepsilon} \nonumber \\
    &\ge \frac{m\cdot\delta\,E_H(y)}{ME_L(u_c)+\varepsilon} \nonumber \\
    &= \frac{m}{M}\cdot\delta\cdot
    \frac{E_H(y)}{E_L(u_c)+\varepsilon/M}.
\end{align}
Since $E_L(u_c)+\varepsilon \ge E_L(u_c)+\varepsilon/M$ due to $M \ge 1$, it follows that:
\begin{equation}
    \mathrm{HL}(\hat y_c)
    \ge
    \frac{m}{M}\cdot\delta\cdot
    \frac{E_H(y)}{E_L(u_c)+\varepsilon}.
\end{equation}
Using the controllable low-frequency assumption A2~\eqref{eq:lf_control}, we further obtain:
\begin{align}
    \mathrm{HL}(\hat y_c) 
    &\ge
    \frac{m}{M}\cdot\delta\cdot
    \frac{E_H(y)}{\kappa E_L(y)+\varepsilon} \nonumber \\
    &=
    \frac{m}{M}\cdot\frac{\delta}{\kappa}\cdot
    \frac{E_H(y)}{E_L(y)+\varepsilon/\kappa}.
\end{align}
Since $\kappa \ge 1$, we have $\varepsilon \ge \varepsilon/\kappa$, which ultimately leads to:
\begin{equation}
    \mathrm{HL}(\hat y_c)
    \ge
    \frac{m}{M}\cdot\frac{\delta}{\kappa}\cdot
    \mathrm{HL}(y).
\end{equation}
This completes the proof.
\end{proof}

\subsection{Detailed Empirical Validation}
\label{supply:valid}

Table~\ref{tab:spectral-analysis} reports empirical statistics for verifying assumptions A1--A3~\eqref{eq:hf_noncontract}--\eqref{eq:bounded_resp}. All results are obtained using the same downstream operator $F$ (FNO), the same test set (CurveVel-A), and an identical frequency-domain decomposition, ensuring a controlled comparison between the Spectral-Preserving DINO Encoder and the bilinear interpolation frontend.

\subsubsection{Metric Definitions}

\noindent\textbf{High-frequency preservation ratio ($\mathrm{Ratio_1}$).}
To evaluate assumption A1~\eqref{eq:hf_noncontract}, we define
\begin{equation}
\label{eq:ratio1}
\mathrm{Ratio_1}
=
\frac{E_H(u_c)}{E_H(y)},
\end{equation}
which directly measures whether the encoder output preserves or contracts high-frequency energy relative to the ground truth.

\noindent\textbf{Low-frequency controllability ratio ($\mathrm{Ratio_2}$).}
To evaluate assumption A2~\eqref{eq:lf_control}, we define
\begin{equation}
\label{eq:ratio2}
\mathrm{Ratio_2}
=
\frac{E_L(u_c)}{E_L(y)}.
\end{equation}
This ratio quantifies the extent to which the low-frequency energy of the encoder output is bounded by that of the ground truth.

\noindent\textbf{Downstream frequency-band gains ($g_H$, $g_L$).}
To evaluate assumption A3~\eqref{eq:bounded_resp}, we define the empirical gain factors of the downstream operator $F$ on the high- and low-frequency bands as
\begin{equation}
\label{eq:gains}
g_H = \frac{E_H(\hat y)}{E_H(u_c)}, \qquad
g_L = \frac{E_L(\hat y)}{E_L(u_c)}.
\end{equation}
Assumption A3~\eqref{eq:bounded_resp} requires the existence of constants $0 < m \le M < \infty$ such that both $g_H$ and $g_L$ lie within the interval $[m, M]$.

\subsubsection{Empirical Observations}

The Spectral-Preserving DINO Encoder achieves a mean value of $\mathrm{Ratio}_1 = 19.032$, which is significantly larger than $1$, providing empirical support for the high-frequency non-contractiveness assumption A1~\eqref{eq:hf_noncontract}. It also yields a mean value of $\mathrm{Ratio}_2 = 34.904$, indicating that the low-frequency energy of the encoder output remains bounded by a finite multiple of the ground truth, thus supporting assumption A2~\eqref{eq:lf_control}.

For the downstream operator, the empirical ranges of $g_H$ and $g_L$ under the encoder frontend are of comparable orders of magnitude and exhibit substantial overlap, supporting the existence of finite constants $m$ and $M$ in assumption A3~\eqref{eq:bounded_resp}. In contrast, the interpolation frontend yields $\mathrm{Ratio}_1 = 3.876 \times 10^{-6}$ and exhibits a strong imbalance between $g_H$ and $g_L$, indicating severe high-frequency collapse and low-frequency-biased amplification.

\begin{table*}[!t]
\centering
\begin{adjustbox}{max width=\textwidth}
\begin{tabular}{llcccccccccc}
\toprule
\multirow{2}{*}{\textbf{Category}} & \multirow{2}{*}{\textbf{Hyperparameter}} & \multicolumn{2}{c}{\textbf{FlatVel}} & \multicolumn{2}{c}{\textbf{CurveVel}} & \multicolumn{2}{c}{\textbf{FlatFault}} & \multicolumn{2}{c}{\textbf{CurveFault}} & \multicolumn{2}{c}{\textbf{Style}} \\
\cmidrule(lr){3-4} \cmidrule(lr){5-6} \cmidrule(lr){7-8} \cmidrule(lr){9-10} \cmidrule(lr){11-12}
& & \textbf{A} & \textbf{B} & \textbf{A} & \textbf{B} & \textbf{A} & \textbf{B} & \textbf{A} & \textbf{B} & \textbf{A} & \textbf{B} \\
\midrule
\multirow{7}{*}{\shortstack[l]{Optimization\\(AdamW)}} 
& Initial LR (LR, $\times 10^{-3}$) & 0.1 & 0.1 & 0.1 & 0.1 & 0.1 & 0.1 & 0.1 & 0.1 & 0.1 & 0.1 \\
& Weight Decay & 0.05 & 1e-4 & 1e-4 & 0.05 & 1e-4 & 0.05 & 1e-4 & 1e-4 & 0.05 & 0.05 \\
& Batch Size & 32 & 32 & 32 & 32 & 32 & 32 & 32 & 32 & 32 & 32 \\
& Training Epochs & 160 & 160 & 160 & 160 & 160 & 160 & 160 & 160 & 160 & 160 \\
& Warmup Epochs & 5 & 5 & 5 & 5 & 5 & 5 & 5 & 5 & 5 & 5 \\
& $T_0$ / $T_{mult}$ & 10 / 2 & 10 / 2 & 10 / 2 & 10 / 2 & 10 / 2 & 10 / 2 & 10 / 2 & 10 / 2 & 10 / 2 & 10 / 2 \\
& Scheduler $\gamma$ & 0.3 & 0.3 & 0.3 & 0.3 & 0.3 & 0.3 & 0.3 & 0.3 & 0.3 & 0.3 \\
\midrule
\multirow{5}{*}{\shortstack[l]{Loss\\Weights}} 
& Grad L1 ($\lambda_{grad}$) & 0.15 & 0.15 & 0.15 & 0.15 & 0.15 & 0.15 & 0.15 & 0.15 & 0.15 & 0.15 \\
& Fourier Mag L1 ($\lambda_{freq}$) & 0.10 & 0.10 & 0.10 & 0.10 & 0.10 & 0.10 & 0.10 & 0.10 & 0.10 & 0.10 \\
& Load Balance ($\lambda_{ce}$) & 0.20 & 0.20 & 0.20 & 0.20 & 0.20 & 0.20 & 0.20 & 0.20 & 0.20 & 0.20 \\
& Router L1 ($\lambda_{l1}$) & 0.60 & 0.60 & 0.60 & 0.60 & 0.60 & 0.60 & 0.60 & 0.60 & 0.60 & 0.60 \\
& Router L2 ($\lambda_{l2}$) & 0.40 & 0.40 & 0.40 & 0.40 & 0.40 & 0.40 & 0.40 & 0.40 & 0.40 & 0.40 \\
\midrule
\multirow{5}{*}{\shortstack[l]{Architecture\\(SPAMoE)}} 
& Hidden / Enc Channels & 96/128 & 128/128 & 128/128 & 96/128 & 128/128 & 96/128 & 128/128 & 128/128 & 96/128 & 96/128 \\
& Top-$k$ Experts & 2 & 2 & 2 & 2 & 2 & 2 & 2 & 2 & 2 & 2 \\
& Band Sharpness ($\gamma$) & 20.0 & 20.0 & 20.0 & 20.0 & 20.0 & 20.0 & 20.0 & 20.0 & 20.0 & 20.0 \\
& Freq. Affinity ($\eta$) & 10.0 & 10.0 & 10.0 & 10.0 & 10.0 & 10.0 & 10.0 & 10.0 & 10.0 & 10.0 \\
& Backbone & \multicolumn{10}{c}{ViT (DINOv3)} \\
\midrule
\multirow{7}{*}{\shortstack[l]{Expert\\Specs}} 
& FNO Modes $(H, W)$ & (16, 16) & (16, 16) & (16, 16) & (16, 16) & (16, 16) & (16, 16) & (16, 16) & (16, 16) & (16, 16) & (16, 16) \\
& FNO Layers & 4 & 4 & 4 & 4 & 4 & 4 & 4 & 4 & 4 & 4 \\
& MNO Scales & 3 & 3 & 3 & 3 & 3 & 3 & 3 & 3 & 3 & 3 \\
& MNO Scale Factors & \multicolumn{10}{c}{1.0, 0.6, 0.3} \\
& MNO Layers & 3 & 3 & 3 & 3 & 3 & 3 & 3 & 3 & 3 & 3 \\
& LNO Modes $(H, W)$ & (16, 16) & (16, 16) & (16, 16) & (16, 16) & (16, 16) & (16, 16) & (16, 16) & (16, 16) & (16, 16) & (16, 16) \\
& LNO Layers & 3 & 3 & 3 & 3 & 3 & 3 & 3 & 3 & 3 & 3 \\
\midrule
\multirow{2}{*}{Data Spec} 
& Input Size ($T \times R$) & \multicolumn{10}{c}{$1k \times 350$} \\
& Output Size ($H \times W$) & \multicolumn{10}{c}{$70 \times 70$} \\
\bottomrule
\end{tabular}
\end{adjustbox}
\caption{Updated hyperparameter configurations for SPAMoE across 10 OpenFWI sub-datasets, incorporating the ViT backbone and latest training settings.}
\label{tab:hyperparams}
\end{table*}

\begin{table*}[t]
\centering
\begin{adjustbox}{max width=\textwidth}
\begin{tabular}{llcccccccccc}
\toprule
\multirow{2}{*}{\textbf{Category}} & \multirow{2}{*}{\textbf{Hyperparameter}} & \multicolumn{2}{c}{\textbf{FlatVel}} & \multicolumn{2}{c}{\textbf{CurveVel}} & \multicolumn{2}{c}{\textbf{FlatFault}} & \multicolumn{2}{c}{\textbf{CurveFault}} & \multicolumn{2}{c}{\textbf{Style}} \\
\cmidrule(lr){3-4} \cmidrule(lr){5-6} \cmidrule(lr){7-8} \cmidrule(lr){9-10} \cmidrule(lr){11-12}
& & \textbf{A} & \textbf{B} & \textbf{A} & \textbf{B} & \textbf{A} & \textbf{B} & \textbf{A} & \textbf{B} & \textbf{A} & \textbf{B} \\
\midrule
\multirow{3}{*}{\shortstack[l]{Optimization\\(AdamW)}} 
& Initial LR (LR, $\times 10^{-3}$) & 0.1 & 0.1 & 0.1 & 0.1 & 0.1 & 0.1 & 0.1 & 0.1 & 0.1 & 0.1 \\
& Weight Decay & 0.05 & 0.05 & 0.05 & 1e-4 & 0.05 & 0.05 & 0.05 & 0.05 & 0.05 & 0.05 \\
& Batch Size / Epochs & 32 / 200 & 32 / 200 & 32 / 200 & 32 / 150 & 32 / 200 & 32 / 200 & 32 / 200 & 32 / 200 & 32 / 200 & 32 / 200 \\
\midrule
\multirow{3}{*}{\shortstack[l]{Architecture\\(Baseline)}} 
& Hidden Channels ($C$) & 64 & 64 & 64 & 64 & 64 & 64 & 64 & 64 & 64 & 64 \\
& FNO Modes $(H, W)$ & (16, 16) & (16, 16) & (16, 16) & (16, 16) & (16, 16) & (16, 16) & (16, 16) & (16, 16) & (16, 16) & (16, 16) \\
& FNO Layers & 8 & 8 & 8 & 8 & 8 & 8 & 8 & 8 & 8 & 8 \\
\midrule
\multirow{1}{*}{Data Spec} 
& Input / Output Resolution & \multicolumn{10}{c}{$1000 \times 350 \rightarrow 70 \times 70$} \\
\bottomrule
\end{tabular}%
\end{adjustbox}
\caption{Comprehensive hyperparameter configurations for SPAMoE across 10 OpenFWI sub-datasets. This table covers optimization settings, loss weight distributions, core model architecture, and expert-specific parameters. }
\label{tab:baseline_params}
\end{table*}


\section{Implementation Details}
This section presents a comprehensive overview of the experimental setup, covering benchmark datasets, evaluation metrics,
and implementation details to ensure a rigorous and reproducible analysis.

\subsection{Training Detail}

\label{sec:training-details}

We train one independent model per OpenFWI 2D sub-dataset (CurveVel-A/B, FlatVel-A/B, CurveFault-A/B, FlatFault-A/B, and Style-A/B). All models share the same architecture and loss design, while optimization hyperparameters follow a unified configuration (Table~\ref{tab:hyperparams}). Training is conducted on dual RTX~4090 GPUs with a per-GPU batch size of 32. Unless otherwise stated, all settings below apply to all sub-datasets.

We report three standard image-level reconstruction metrics on the OpenFWI test sets: mean absolute error (MAE), root mean squared error (RMSE), and peak signal-to-noise ratio (PSNR). MAE and RMSE are computed between the predicted and ground-truth velocity maps (lower is better), while PSNR measures reconstruction fidelity in decibels (higher is better).

\subsection{Hyperparameters Details}
\label{supply:hyper}

For completeness and reproducibility, Table~\ref{tab:hyperparams} provides detailed hyperparameter configurations for SPAMoE across the ten OpenFWI sub-datasets, covering optimization settings, loss weights, model architecture, and expert-specific parameters. Table~\ref{tab:baseline_params} lists the hyperparameter configurations of the baseline (Only FNO) model. 


\subsection{Evaluation Metric}
For the FWI task, we strictly follow the evaluation protocol of the OpenFWI benchmark. Specifically, we adopt exactly the same metric definitions and implementation code as provided in the OpenFWI open-source repository, ensuring fair and directly comparable evaluations. The employed metrics include Mean Absolute Error (MAE), Root Mean Squared Error (RMSE), and Structural Similarity Index (SSIM)~\cite{wang2004ssim}.

For the pipe-flow task, we use the relative $\ell_2$ error, consistent with the evaluation setting of LaMO~\cite{tiwarilatent}, enabling a fair comparison with prior work.

\section{Visualization}
\subsection{Additional Main Result}
\begin{figure*}[t]
    \centering
    \includegraphics[width=\linewidth]{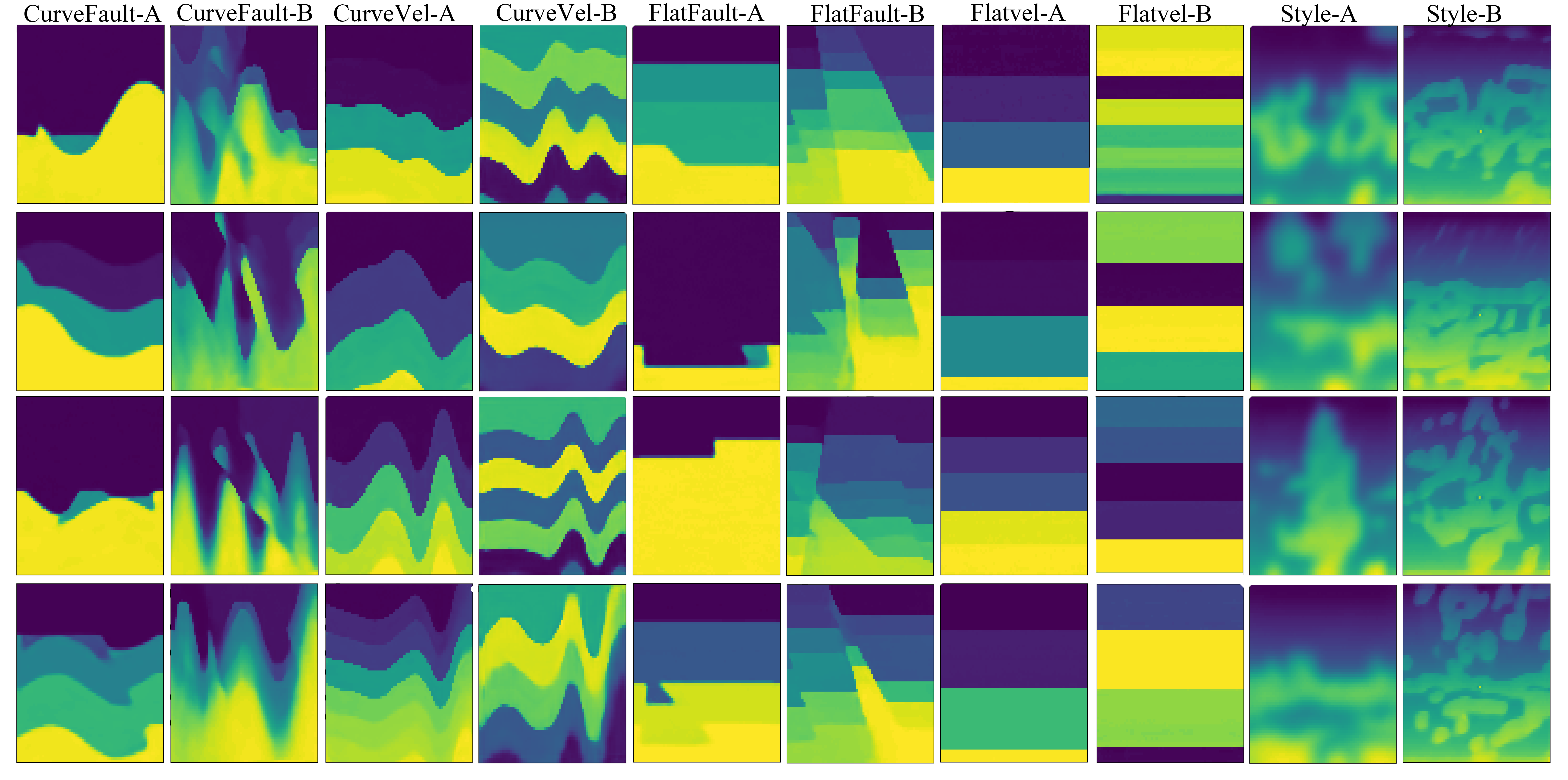}
    \caption{ Additional qualitative results of our framework on OpenFWI. Each column corresponds to one of the ten OpenFWI sub-datasets, visualizing reconstructed velocity models not shown in the main paper. }
    \label{fig:mainresult_supply}
\end{figure*}

To complement the qualitative results presented in the main paper, we provide additional reconstructed velocity models produced by our framework on OpenFWI in Figure~\ref{fig:mainresult_supply}. These examples are included for completeness and are not shown in the main text due to space limitations. 

For each sub-dataset, we select four representative velocity models predicted by our method for qualitative visualization. As shown, the proposed framework can stably reconstruct a variety of typical geological structures, including smooth varying background layers, curved stratified formations, and clear fault interfaces. For regions with highly mixed structural components and strong spatial variability, a small number of examples exhibit locally smoother boundary transitions or reduced fine-scale contrast, reflecting the inherent difficulty of disentangling multi-scale features in such complex scenarios.

\subsection{Visualization of Intermediate Representations}
\begin{figure*}[t]
    \centering
    \includegraphics[width=\linewidth]{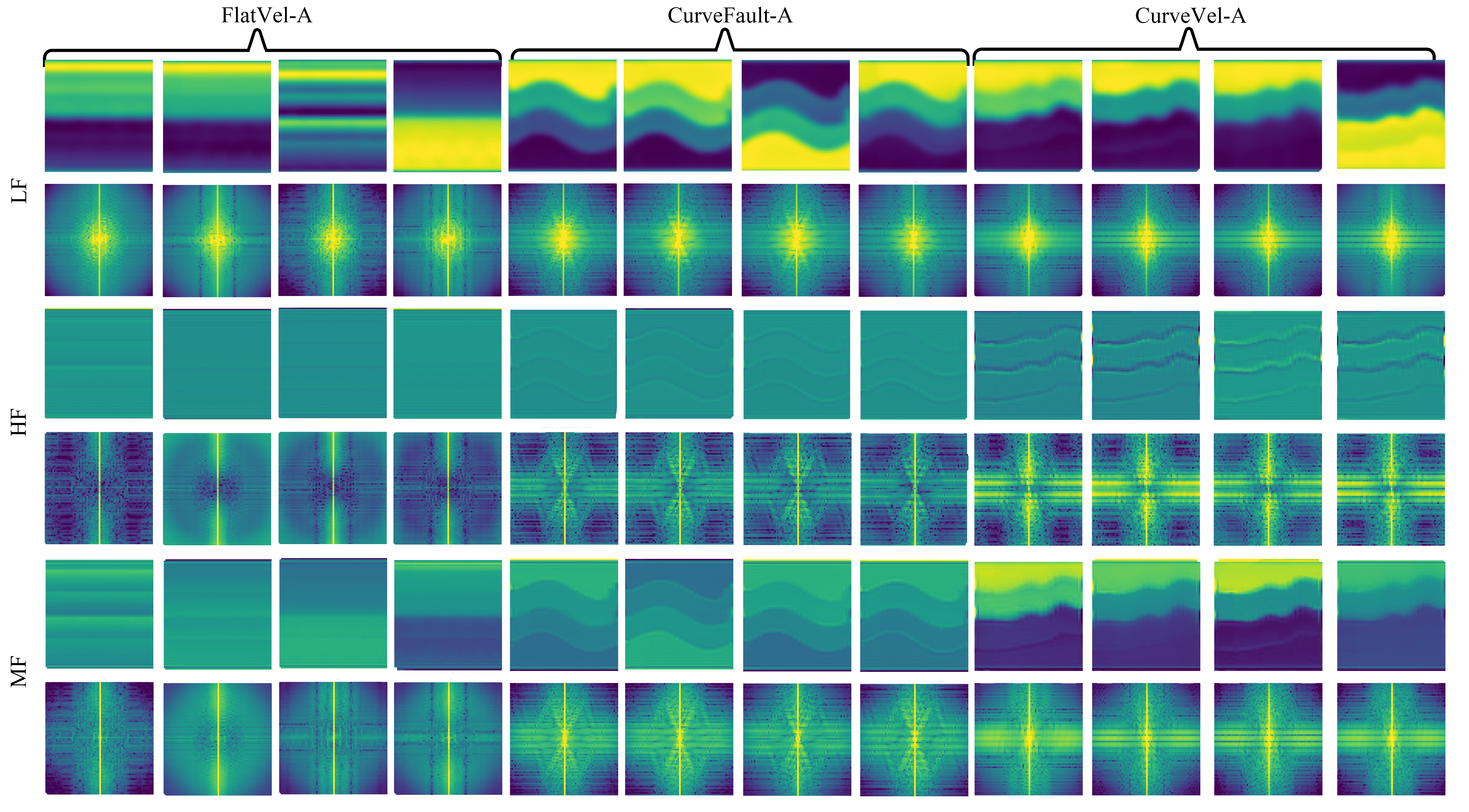}
    \caption{ Total 12 columns show 4 samples per sub-dataset (from left to right: FlatVel-A, CurveFault-A, CurveVel-A). Rows from top to bottom visualize the low-, high-, and medium-frequency components partitioned by the Encoder, paired with their 2D amplitude spectrums. This demonstrates the model's ability to decouple physical features before MoE expert selection. }
    \label{fig:xiaorong}
\end{figure*}

To verify the effectiveness of the spectral partitioning module, we analyze the intermediate features captured prior to the routing stage (Figure~\ref{fig:xiaorong}). The visualization confirms that the module successfully decouples the input into low-, high-, and medium-frequency paths, thereby providing appropriate inputs for the subsequent MoE experts.

Inspecting the decoupled low-, high-, and mid-frequency visualizations and their corresponding spectra, we observe that the low-frequency components predominantly capture large-scale background variations and smooth stratified trends, which are responsible for the global velocity distribution and long-wavelength structures. In contrast, high-frequency components emphasize sharp discontinuities, fine-scale layer boundaries, and fault-related features, exhibiting concentrated energy in the outer spectral regions. The mid-frequency components mainly represent transitional structures between these two components, such as moderately varying layers and curved interfaces, bridging global context and local details. Consequently, this spectral partitioning provides structurally and spectrally complementary representations, enabling each expert to focus on the frequency band most relevant to its modeling capacity.


\end{document}